\documentclass[10pt,twocolumn,letterpaper]{article}

\usepackage{3dv}
\usepackage{times}
\usepackage{epsfig}
\usepackage{graphicx}
\usepackage{amsmath}
\usepackage{amssymb}

\usepackage{amsfonts}		%
\usepackage{pifont}%
\usepackage{mathtools, nccmath}
\usepackage{booktabs}       %
\usepackage{multirow}
\usepackage{enumitem}
\usepackage{graphicx}
\usepackage{comment}
\usepackage[table, usenames, dvipsnames]{xcolor}
\usepackage{makecell}
\usepackage{arydshln}   %
\usepackage{latexsym}
\usepackage{caption}
\usepackage{ragged2e}
\usepackage[capbesideposition=outside,capbesidesep=quad]{floatrow}		%

\usepackage{blindtext}
\usepackage{enumitem}
\usepackage{algorithm}
\usepackage{algpseudocode}

\usepackage[pagebackref=true,breaklinks=true,colorlinks,bookmarks=false]{hyperref}

\threedvfinalcopy %

\def\threedvPaperID{318} %

\ifthreedvfinal\fi

\graphicspath{{images/}}

\newtheorem{theorem}{Theorem}[section]

\newtheorem{proposition}[theorem]{Proposition}

\newtheorem{definition}[theorem]{Definition}

\newtheorem{example}[theorem]{Example}

\newenvironment{proof}[1][Proof]{\begin{trivlist}
\item[\hskip \labelsep {\bfseries #1}]}{\end{trivlist}}
\newcommand{\xmark}{\ding{55}}%
\newcommand{\noo}{\textcolor{red}{\xmark}}
\newcommand{\yes}{\textcolor{OliveGreen}{\checkmark}}

\newcommand{\norm}[1]{\left\lVert#1\right\rVert}
\newcommand{\qed}{\nobreak \ifvmode \relax \else
      \ifdim\lastskip<1.5em \hskip-\lastskip
      \hskip1.5em plus0em minus0.5em \fi \nobreak
      \vrule height0.75em width0.5em depth0.25em\fi}

\newcommand{\algrule}[1][.2pt]{\par\vskip.5\baselineskip\hrule height #1\par\vskip.5\baselineskip}

\makeatletter
\def\thickhline{%
  \noalign{\ifnum0=`}\fi\hrule \@height \thickarrayrulewidth \futurelet
   \reserved@a\@xthickhline}
\def\@xthickhline{\ifx\reserved@a\thickhline
               \vskip\doublerulesep
               \vskip-\thickarrayrulewidth
             \fi
      \ifnum0=`{\fi}}
\makeatother

\makeatletter
\def\thickhline{%
  \noalign{\ifnum0=`}\fi\hrule \@height \thickarrayrulewidth \futurelet
   \reserved@a\@xthickhline}
\def\@xthickhline{\ifx\reserved@a\thickhline
               \vskip\doublerulesep
               \vskip-\thickarrayrulewidth
             \fi
      \ifnum0=`{\fi}}
\makeatother

\newlength{\thickarrayrulewidth}
\DeclarePairedDelimiter\abs{\lvert}{\rvert}%

\makeatletter
\let\oldabs\abs
\def\abs{\@ifstar{\oldabs}{\oldabs*}}

\makeatletter
\renewcommand*\env@matrix[1][*\c@MaxMatrixCols c]{%
  \hskip -\arraycolsep
  \let\@ifnextchar\new@ifnextchar
  \array{#1}}
\makeatother

\newlist{steps}{enumerate}{2}
\setlist[steps, 1]{label=(\arabic*) , font=\bfseries,  wide=0pt}%
\setlist[steps, 2]{label=\emph{alph*}),  wide=0pt, before=\leavevmode, topsep=0pt}%

\newlist{modellist}{enumerate}{2}
\setlist[modellist, 1]{label=Model~(\arabic*) -- , font=\bfseries,  wide=0pt}%
\setlist[modellist, 2]{label=\emph{alph*}) --,  wide=0pt, before=\leavevmode, topsep=0pt}%

\setlist[enumerate]{itemsep=0mm}

\newcolumntype{R}[2]{%
    >{\adjustbox{angle=#1,lap=\width-(#2)}\bgroup}%
    l%
    <{\egroup}%
  }

\newcommand{\boldparagraph}[1]{\vspace{0.05em}\noindent{\bf #1} }

\newcommand{\Mmat}{\mathsf{M\mathstrut}}

\begin{document}

\title{Unsupervised Monocular Depth Reconstruction of Non-Rigid Scenes}
\author{\hspace{-1.3em}
Ay\c{c}a Takmaz$^2$\!,
Danda Pani Paudel$^1$\!,
Thomas Probst$^1$\!,
Ajad Chhatkuli$^1$\!,
Martin R. Oswald$^{2,3}$\!,
Luc Van Gool$^{1,4}$\\
$^1$Computer Vision Lab, ETH Zurich \qquad
$^2$Department of Computer Science, ETH Zurich\\
$^3$University of Amsterdam, Netherlands \qquad $^4$VISICS, ESAT/PSI, KU Leuven, Belgium\\ \hspace{-1.4em}
{\tt\small takmaza@ethz.ch}, 
{\tt\small \{paudel,probstt,ajad.chhatkuli,vangool\}@vision.ee.ethz.ch},
{\tt\small moswald@inf.ethz.ch}
}

\maketitle
\thispagestyle{empty}

\begin{abstract}
Monocular depth reconstruction of complex and dynamic scenes is a highly challenging problem.
While for rigid scenes learning-based methods have been offering promising results even in unsupervised cases, there exists little to no literature addressing the same for dynamic and deformable scenes. In this work, we present an unsupervised monocular framework for dense depth estimation of dynamic scenes, which jointly reconstructs rigid and non-rigid parts without explicitly modelling the camera motion. Using dense correspondences, we derive a training objective that aims to opportunistically preserve pairwise distances between reconstructed 3D points.
In this process, the dense depth map is learned implicitly using the as-rigid-as-possible hypothesis. Our method provides promising results, demonstrating its capability of reconstructing 3D from challenging videos of non-rigid scenes. Furthermore, the proposed method also provides unsupervised motion segmentation results as an auxiliary output.

\end{abstract} %

\vspace{-2mm}
\vspace{-3mm}
\section{Introduction}\label{sec:intro}
Understanding the 3D structure of a scene can provide important cues for many tasks such as robot navigation \cite{robotnav}, motion capture \cite{mocap}, scene understanding~\cite{kim20133d}, and augmented reality \cite{augmreality}. While humans are exceptionally capable of inferring the non-rigid 3D structures from an image, geometric computer vision techniques either require a large amount of labeled data or are capable of learning only from rigid scenes~\cite{zhou}. However, in the real world, scenes often consist of non-rigid and dynamic elements. Thus, it is quite natural to seek for the ability of inferring the depth from an image of any given scene, regardless of whether it is highly dynamic or not. We refer to Fig.~\ref{fig:teaser} for some examples. %

The ability of humans to understand their environment geometrically and semantically is mainly acquired through childhood learning.
In an attempt to emulate this ability, learning-based strategies have been applied to the problem of \textit{depth from single-view}. 
In fact, many learning-based methods already offer very promising progress in this direction. Among them, supervised methods ~\cite{liu, wild-nips, bastian} aim to reconstruct the depth of rigid and non-rigid parts during training, whereas many unsupervised methods~\cite{rescue, zhou, geonet, digging, yuhua, wicked, directmethods, yuille, consistent, packnet} mainly employ a training strategy which aims to reconstruct the rigid parts of a scene. %
Although a recent unsupervised method ~\cite{google-depth} also reconstructs objects translating on a ground plane, it has limitations in terms of modelling highly non-rigid scenes. 
Other unsupervised methods for non-rigid reconstruction exploit object-specific priors~\cite{ranjan2018generating,wu2020unsupervised,Yang_2021_CVPR}, or reconstruct sparse~\cite{novotny2019c3dpo} or dense~\cite{Sidhu2020, Sridhar2019} points of a single non-rigid object. Such methods, however, do not have the same applicability or motivation as that of single view scene depth estimation. Unlike the monocular setup, the calibrated stereo (or multi-camera) methods for learning single view depth~\cite{left-right,mehta2018structured} do not suffer from ambiguities due to non-rigidity and can handle complete scene depth. However, they have significant limitations, particularly when acquiring training data via calibrated stereo pairs is not feasible. 
\begin{figure}[t]
  \centering
    \includegraphics[width=\columnwidth]{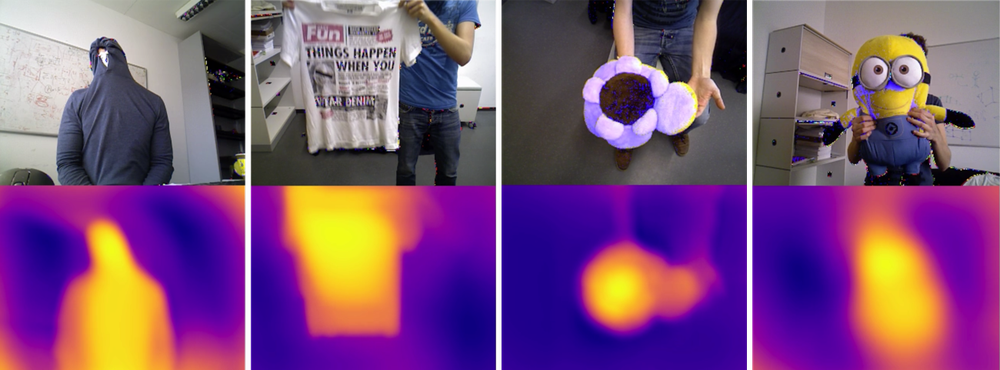}
    \vspace{-16pt}
    \caption{\textbf{Unsupervised Depth.} Image-depth pairs; depth is estimated in an unsupervised manner using the proposed method.
    \label{fig:teaser}} %
\end{figure}

In this work, we are interested in learning to predict dense depth from a single image using an unsupervised monocular pipeline. 
More importantly, we would like to reconstruct the depth of all parts of the scene during the training, irrespective of the rigidity/non-rigidity of these parts. 
Our unsupervised setup assumes that only calibrated monocular videos~\footnote{Our method can potentially also be used for multi-view setups.} with known intrinsics are available during training. 
Such an assumption is realistic as well as crucial for a wide variety of setups, ranging from consumer to surgical cameras, where depth or stereo acquisition for supervision is often impractical. 
In this context, learning depth from monocular videos of non-rigid scenes remains an unresolved problem. 
This is no surprise, given the challenges, \eg, ill-posedness, ambiguities, inconsistent priors. %

The success of the unsupervised depth learning methods for rigid scenes can be primarily attributed to the advancements in deep learning and the rigid reconstruction constraints used in such methods. 
This motivates us to explore the non-rigid 3D reconstruction literature employing various assumptions. Our goal is to keep in mind an overview of the literature and to use the gained insights to build our unsupervised monocular pipeline for depth reconstruction in non-rigid scenes.  
Our main \textbf{contributions} are threefold:
\begin{itemize}[itemsep=-2pt,topsep=3pt,leftmargin=*]
\item We reformulate the Non-Rigid Structure-from-Motion (NRSfM) priors in a novel unified framework using the Euclidean distance matrix measures across views. This contribution is summarized in Table~\ref{tab:related_Work_overview}.
\item  The utility of the proposed framework is also demonstrated for unsupervised non-rigid monocular depth, by exploiting the as-rigid-as possible (ARAP) prior during CNN training. For the implementation of the ARAP prior, we define and employ a concept of \textit{motion embeddings}. This contribution is illustrated in Fig.~\ref{fig:unsupervisedEmbeddingModel-met}.
\item Through experiments, we provide interesting new insights towards learning non-rigid scene depths in an unsupervised manner, detailed in our discussion section. 
\end{itemize}

\section{Non-Rigid Reconstruction Revisited}\label{sec:revisit}
Since generic non-rigid 3D reconstruction from a monocular camera is an ill-posed problem, methods in the literature rely on some priors or assumptions about the scene.
The most common scene priors can be broadly divided into the following four categories.

\boldparagraph{Low-Rank (LR).} 
The low-rank prior assumes that non-rigid 3D structures can be expressed as a linear combination of finite basis shapes. 
The landmark work of~\emph{Costeira \& Kanade} ~\cite{costeira1998multibody} developed for orthographic cameras (or slight variations) uses the LR prior for multi-body 3D scene reconstruction. Since, it been widely used in various cases for the 3D reconstruction of single~\cite{bregler2000recovering,brand2001flexible,bartoli2008coarse,olsson2009convex,dai2012simple,fragkiadaki2014grouping,khan2014non,valmadre2015closed,agudo2018image,novotny2019c3dpo,kong2019} and multiple ~\cite{xiao2005uncalibrated,akhter2010trajectory} non-rigid objects.

\boldparagraph{Scene Motion (SM).} Several notable works of \emph{Shashua et al.}~\cite{shashua2000homography,avidan2000trajectory,wolf2002projection,vogel20153d} have shown that the known planar/linear motion prior can be exploited to reconstruct the 3D structure of dynamic scenes\footnote{The recent work~\cite{google-depth} can be seen as an adaptation of~\cite{avidan2000trajectory}.}. Work of \emph{Ozden et al.}~ \cite{ozden2004} also tackled dynamic scene 3D reconstruction using the principle of non-accidentalness.
An insightful work of \emph{Hartley and Vidal}~\cite{hartley2008perspective} reveals that the method of~\cite{wolf2002projection} can be indeed extended to the generic non-rigid case, using the low-rank structure and without an explicit motion prior, with only one \emph{severe} ambiguity, which stems from the fact that the recovered camera motion is relative to the moving points. This in turn implies the unfavorable news that the low-rank prior alone is not sufficient to recover scale-consistent non-rigid 3D structures from monocular projective cameras, when the objects are independently moving in the multi-body setting.

\boldparagraph{Isometric Deformation (ID).} To avoid the algebraic prior of LR, \emph{Salzmann et al.}~\cite{salzmann2010linear} and \emph{Perriollat et al.}~\cite{perriollat2011monocular} %
introduced the geometric `ID' prior of an object deforming isometrically, which makes non-rigid surface reconstruction possible when used with a known template. Since, the ID prior of the object has been used in many other works~\cite{ shen2009monocular,vicente2012soft,chhatkuli2014non,parashar2016isometric, chhatkuli2016inextensible,wang2016template,probst2018incremental,li2010multi}. \emph{Taylor et al.}~\cite{taylor2010} introduced local-rigidity for non-rigid reconstruction, which is a form of the ID prior. Other related geometric priors have also been exploited in the literature~\cite{malti2013monocular,russell2014video,agudo2014good}. When objects undergo severe deformations, relying on a geometric prior for the objects and avoiding any explicit camera motion estimation have demonstrated state-of-the art results for non-rigid object reconstruction~\cite{parashar2020local,parashar2020robust,probst2018incremental,gallardo2020shape,malti2017elastic}. The benefit of not estimating the camera motion explicitly can be understood quite intuitively, as the estimated rigid camera motion is merely a motion with respect to the scene. %
Interested readers may recall the relevant work in~\cite{hartley2008perspective}.
In this regard, the works of \emph{Li}~\cite{li2010multi}, \emph{Ji et al.}~\cite{maxrig} and \emph{Chhatkuli et al.}~\cite{chhatkuli2016inextensible} stand out. 
While reconstructing inextensible structures, \cite{chhatkuli2016inextensible} formulates an explicit-motion-free problem, which turns out to be equivalent to the rigid case formulation of~\cite{li2010multi} in the absence of deformations. 
In fact, the formulation of~\cite{chhatkuli2016inextensible} is shown to be effective for both rigid and non-rigid scenes. All three works \cite{chhatkuli2016inextensible,maxrig,li2010multi} are based on the pairwise distance equality in 3D, while \cite{li2010multi} uses the same pairwise sampling as \cite{maxrig}.

\boldparagraph{As-Rigid-As-Possible (ARAP).} Non-rigid shape modelling using the ARAP assumption is very prevalent in computer graphics~\cite{alexa2000rigid,igarashi2005rigid,sorkine2007rigid}. 
The ARAP assumption maximizes rigidity while penalising stretching, shearing, and compression. ARAP concept has also been used in many works~\cite{parashar2015rigid,kumar2019dense} for monocular non-rigid 3D reconstruction. In particular, two works of~\emph{Parashar et al.}~\cite{parashar2015rigid} and \emph{Kumar et al.}~\cite{kumar2019dense} are noteworthy. Using the shape template of an object, ~\cite{parashar2015rigid} reconstructs the deformed volumetric 3D under the ARAP assumption. On the other hand,~\cite{kumar2019dense} demonstrates the practicality of ARAP in depth densification and refinement for non-rigid \emph{scenes}, using multi-view setups. More interestingly to us, the ARAP prior is sufficient to resolve the scale ambiguity between freely moving parts, \ie the issue previously discussed by  \emph{Hartley \& Vidal} in~\cite{hartley2008perspective}.  %

In summary, the computation of camera motion becomes ambiguous in complex non-rigid scenes for a monocular camera. Nevertheless, the structure of both rigid scenes and non-rigid objects\footnote{Given that the assumed object's prior is sufficient.} can still be recovered without any explicit camera motion estimation. On the other hand, the scale consistent reconstruction intrinsically requires some additional prior, such as the ARAP assumption used in this work. 
\section{Preliminaries}
\boldparagraph{Notations.}
We denote matrices with uppercase and their elements with double-indexed lowercase letters: $\mathsf{A}=(a_{ij})$. Similarly, we write vectors and index them as: $\mathsf{a}=(a_{i})$. The inequality $\mathsf{A} > 0$  refers to $(a_{ij}>0)$, unless mentioned otherwise.
We use special uppercase Latin or Greek letters for sets and graphs, such as $\mathcal{S}$ and $\mathcal{G}$. 
The lowercase Latin letters, as in $a$, are used for scalars. The set of neighbours of $i$ within a radius $r$ is given by $\mathcal{N}_r(i)$. %
\begin{figure}[tb]
  \centering
    \includegraphics[width=0.9\columnwidth,trim={0 150 0 110},clip]{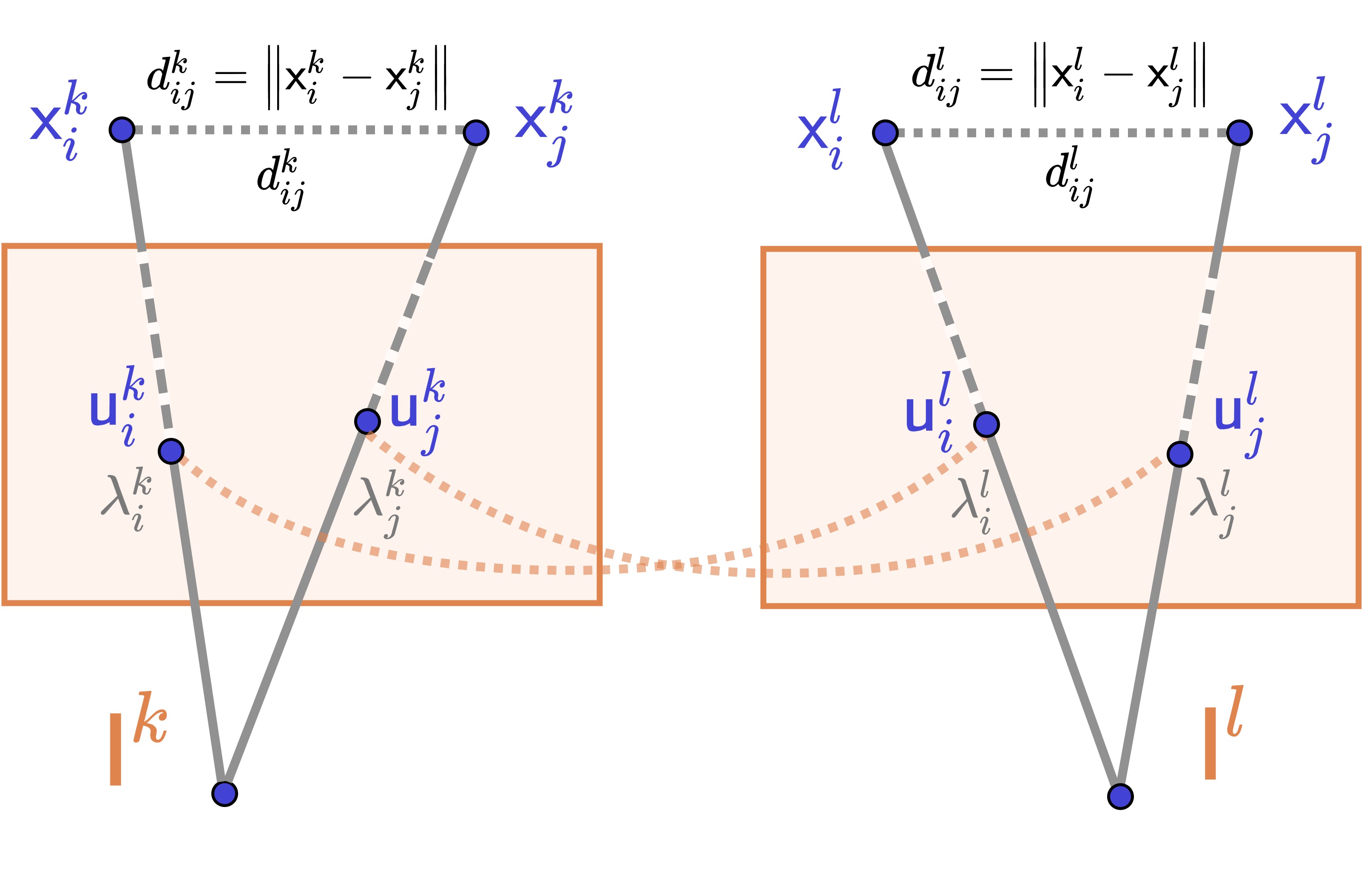}
    \vspace{-1em}
    \caption{\textbf{Pairwise distance formulation.} The pair of 3D Euclidean distances in $k$-th and $l$-th views are given by   ${d_{ij}^k=\norm{\mathsf{x}_i^k-\mathsf{x}_j^k}}$ and ${d_{ij}^l=\norm{\mathsf{x}_{i}^l-\mathsf{x}_{j}^l}}$, respectively.   
    Every 3D point, say  $\mathsf{x}_i^k = \lambda_i^k\mathsf{u}_i^k$, is expressed using the unknown depth and known homogeneous image coordinates, \ie $\lambda_i^k$ and  $\mathsf{u}_i^k$.
    \label{fig:notations1}}
\end{figure}
    
\boldparagraph{Non-rigid Structure-from-Motion (NRSfM).} Our problem formulation is camera motion independent, similar\footnote{Similar also in mathematical terms.} to~\cite{li2010multi,salzmann2010linear,chhatkuli2016inextensible}.
We pose the NRSfM problem as finding point-wise depth in each view. We use superscript $k$ to denote the $k$-th image (among $m$) and subscript $i$ to denote the $i$-th point (among $n$). As illustrated in Fig.~\ref{fig:notations1}, we represent the unknown depth as $\lambda_i^k$ and the known homogeneous image coordinates as $\mathsf{u}_i^k$.  The set of points in the $k$-th view is given as $\mathcal{X}^k = \{X_i^k| \forall~i~\in~1\dots n\}$. The Euclidean distance between point $X_i^k$ and point $X_j^k$ is denoted as $d_{ij}^k$. In our formulation, we use the Euclidean distance matrix:
\begin{definition}[Euclidean distance matrix] Euclidean Distance Matrix (EDM), say $\mathsf{E}\in\mathbb{R}^{n\times n}$, is a matrix representing the spacing of $n$ points in $3$-dimensional Euclidean space, say ${\mathsf{X}=[\mathsf{x}_1,\mathsf{x}_2,\ldots,\mathsf{x}_n]\in\mathbb{R}^{3\times n}}$, and the entries of $\mathsf{E}$ are given by, $e_{ij}=d_{ij}^2 =\norm{\mathsf{x}_i-\mathsf{x}_j}^2$.
\end{definition}
Let $\mathsf{\Lambda^k}\!=\![\lambda_1^k, \lambda_2^k, \ldots, \lambda_n^k] \!\in\!\mathbb{R}^{n}$,  be the sought depth of the $k$-th view. 
We represent the 3D structure in the form, 
${\mathsf{X}^k(\mathsf{\Lambda^k})=[\lambda_1^k\mathsf{u}_1^k,\lambda_2^k\mathsf{u}_2^k,\ldots,\lambda_n^k\mathsf{u}_n^k]\in\mathbb{R}^{3\times n}}$. Let us define the Gram matrix $\mathsf{G}(\mathsf{\Lambda^k})=(\mathsf{X}^k(\mathsf{\Lambda^k}))^\intercal\mathsf{X}^k(\mathsf{\Lambda^k})\in\mathbb{R}^{n\times n}$.  
The EDM for the $k$-th view is then given by:
\begin{equation}
    \mathsf{E}(\mathsf{\Lambda}^k) = \text{diag}(\mathsf{G}(\mathsf{\Lambda}^k))\mathbf{1}^\intercal 
    -2\mathsf{G}(\mathsf{\Lambda}^k)
    +\mathbf{1}\text{diag}(\mathsf{G}(\mathsf{\Lambda}^k))^\intercal.
\end{equation}
NRSfM aims to estimate $\{\mathsf{\Lambda}^k\}$ by imposing priors on $\mathsf{E}(\mathsf{\Lambda}^k)$. For clarity, we present an example of the rigid case. 
\begin{example}[The rigid structure prior] For noise- and outlier-free rigid scenes, the estimated $\{\mathsf{\Lambda}^k, \mathsf{\Lambda}^l\}$ must satisfy, $\mathsf{E}(\mathsf{\Lambda}^k)-\mathsf{E}(\mathsf{\Lambda}^l)=0,$ for all $k,l=1,\ldots, m$. Intuitively, the distance between any pair of points must be preserved across views. Method of~\cite{li2010multi} is a variant of this example. %
\label{ex:rigidExample}
\end{example}
The prior of the above example is also referred as the \emph{global rigidity} constraint~\cite{whiteley2004rigidity,eren2002closing}. Weaker than the \emph{global rigidity} is \emph{local rigidity}. 
Interested readers can find a thorough study of local/global rigidity in~\cite{li2010multi}. 
Here, we are interested in the non-rigid isometric deformation prior~\cite{perriollat2011monocular,taylor2010,salzmann2010linear,chhatkuli2016inextensible}. 
Let us introduce a weight matrix $\mathsf{W}\in\mathbb{R}^{n\times n}$, whose entries are $w_{ij}=1$ if $j\in\mathcal{N}_r(i)$ for some radius $r$, and $w_{ij}=0$ otherwise. Then the relaxed isometric prior of ~\cite{salzmann2010linear,chhatkuli2016inextensible,probst2018incremental} aims to reconstruct the depth by solving,
\begin{align}
\label{eq:iso-nrsfm}
\begin{split}
& \text{find}\quad \{\mathsf{\Lambda}^k\}, \\
& \text{s.t.}\quad  \mathsf{W}\odot\big(\mathsf{E}(\mathsf{\Lambda}^k)-\mathsf{E}(\mathsf{\Lambda}^l)\big)=0, \ \forall k, l,\\
& \quad \quad \, \mathsf{\Lambda}^k > 0, \ \forall k \enspace.
\end{split}
\end{align}
where $\odot$ represents the Hadamard-product. Intuitively,~\eqref{eq:iso-nrsfm} aims to preserve the local Euclidean distances across views. The constraints $\mathsf{\Lambda}^k > 0$ ensure positive depths in all views. 
It is important to note that for sufficiently large radii $r$ the problem of~\eqref{eq:iso-nrsfm} is equivalent to that of the Example~\ref{ex:rigidExample}.
\section{This Work} %
One may use the formulation in Eq.~\eqref{eq:iso-nrsfm} to learn to reconstruct both \emph{rigid} scenes or non-rigid \emph{objects}, for known $\mathsf{W}$\footnote{For example,~\cite{chhatkuli2016inextensible} uses 2D neighbors to construct $\mathsf{W}$.}. We however, are  interested in reconstructing \emph{complex scenes} consisting of both, which requires additional priors (recall Section~\ref{sec:revisit}).
Moreover, it is unclear how to obtain the weight matrix $\mathsf{W}$ in general cases. We use the ARAP assumption to address the scale ambiguity problem. In this process, the weight matrix is also learned along with the depth.
If one aims to exploit other non-rigid priors, later in Section~\ref{subSec:previousWorks} we provide a thorough analysis in that direction.

\subsection{Problem Formulation}
We aim to estimate the depth $\Lambda$ for each view, and the weight matrix $\mathsf{W}$ for given view-pairs.
Our weights $w_{ij}\in[0,1]$ can be interpreted as rigidity scores, between points $X_i$ and $X_j$.
Since, the concept of rigidity is meaningful only for two (or more) views, our rigidity scores are computed accordingly.
We consider two points to be rigidly connected, if their distance does not change across views.
If one seeks for pair-wise rigidity on generic graphs, several interpretations can be derived based on the graph  connectivity~\cite{whiteley2004rigidity,eren2002closing,li2010multi}.
In the context of this paper, we formulate the ARAP assumption as follows: 
 \begin{definition}[{As-rigid-as-possible}]
 For a point set under deformation, the as-rigid-as-possible model assumes that every pair-wise distance of the fully connected graph between points, respects at least some degree of rigidity.
 \end{definition}
 We can now formalize our problem statement as follows.
\begin{align}
\label{eq:deep-iso-nrsfm}
\begin{split}
& \min_{\{\mathsf{\Lambda}^k\}, \{\mathsf{W}^{kl}\}}\quad \eta, \\
& \text{s.t.}\quad  |\mathsf{W}^{kl}\odot\big(\mathsf{E}(\mathsf{\Lambda}^k)-\mathsf{E}(\mathsf{\Lambda}^l)\big)|\leq \eta\norm{\mathsf{W}^{kl}}_{1,1}, \\
& \quad \quad \,1\geq\mathsf{W}^{kl}\geq \tau,\, \, \mathsf{\Lambda}^k > 0, \ \forall k.
\end{split}
\end{align}
Here, the positive scalars $\tau$ and $\eta$ are the rigidity threshold and the rigidity adjusted maximum allowed distance error, respectively.   
Very often the point pairs from non-rigid objects respect rigidity. 
This is when many priors are best justified. 
In general, local rigidity does not imply global rigidity. 
For large fully connected graphs though, pair-wise rigidity means global rigidity. 
We relax this constraint by allowing different edges to have different rigidity scores.
We like to draw the reader's attention on two key (and somewhat related) aspects of our formulation: \textbf{(1)} image pair-wise rigidity scores; and \textbf{(2)} global connectivity.      
 
Our motivation for using individual $\mathsf{W}^{kl}$ stems from the following observation: across all frame pairs, most of the edges respect rigidity at least once in the global sense (e.g. stopping objects or periodic motions).
This global rigidity can be captured and propagated in the rest of the reconstruction for scale consistency, if the global connectivity is established using the fully connected graph. %
For local reconstruction, we can rely on local connectivity which often better respects rigidity.
Note that the maximum allowed adjusted distance error, \ie $\eta$, is measured in the normalized (by the weight matrix) form. 
This encourages rigid edges to bear higher weights, and vice versa.
Such distribution of weights prioritizes the pairs holding rigidity to reconstruct rigidly. 
On the other hand, non-rigid pairs -- if reconstructed correctly at their rigid instant (with highest weight) -- can still satisfy the imposed constraints, as long as the extended distance error does not exceed $\eta/\tau$.

\subsection{Relation to Previous Works}  \label{subSec:previousWorks}

\begin{table}[tb]
  \centering
  \scriptsize
  
  \newcommand{\RotText}[1]{ \rotatebox[origin=l]{90}{\scriptsize #1 } }

  \newcommand{\bldist}{\hspace{8pt}}%
  \setlength{\tabcolsep}{0.1pt}
  \begin{tabular}{l@{\bldist}ccccccc@{\bldist}l}
  	\toprule
  	Method & \RotText{Unsup.-Monoc.} & \RotText{Non-Rigid} &
  	\RotText{Rigid prior} & \RotText{ARAP} & \RotText{Isom. deform.} & \RotText{Scene motion} & \RotText{Low rank} & Constraints\\
  	\midrule
  	\emph{Li}~\cite{li2010multi} 
  	    &  &  & \yes & & & & & $w_{ij}^{kl}=1$  \\[0.8pt] \hdashline \noalign{\vskip 1pt}
  	\emph{Parashar~\etal}~\cite{parashar2015rigid} 
  	    &  & \yes &  & \yes & & & & \multirow{2}{*}{$1\geq w_{ij}^{kl}\geq t \geq 0$}  \\
  	\emph{Kumar~\etal}~\cite{kumar2019dense}
  	    &  & \yes & \yes & \yes & & & & \\[0.8pt] \hdashline \noalign{\vskip 1pt}
  	\emph{Salzmann~\& Fua}~\cite{salzmann2010linear}
  	    &  & \yes &  & & \yes & & & \multirow{4}{*}{\makecell[l]{$w_{ij}^{kl}=1$, if $j\in\mathcal{N}_r(i)$,\\ $w_{ij}^{kl}=0$, otherwise}} \\
  	\emph{Taylor~\etal}~\cite{taylor2010}  
  	    &  & \yes & \yes & & \yes & & & \\
  	\emph{Chhatkuli~\etal}~\cite{chhatkuli2016inextensible} 
  	    &  & \yes &  & & \yes & & & \\
  	\emph{Probst~\etal}~\cite{probst2018incremental} 
  	    &  & \yes &  & & \yes & & & \\[0.8pt] \hdashline \noalign{\vskip 1pt}
  	\emph{Shashua \& Wolf}~\cite{shashua2000homography} 
  	    &  & \yes & \yes & & & \yes & & \multirow{4}{*}{\makecell[l]{$w_{ij}^{kl}=s_e^{kl}$, if $e\in\mathcal{S}^{kl}$,\\  $w_{ij}^{kl}=0$ , otherwise}} \\
  	\emph{Avidan \& Shashua}~\cite{avidan2000trajectory} 
  	    &  &  & \yes & & & \yes & & \\
  	\emph{Vogel~\etal}~\cite{vogel20153d} 
  	    &  &  & \yes & & & \yes & & \\
  	\emph{Li~\etal}~\cite{google-depth} 
  	    & \yes & \yes & \yes & & & \yes & & \\[0.8pt] \hdashline \noalign{\vskip 1pt}
  	\emph{Costeira~\etal}~\cite{costeira1998multibody} 
  	    &  &  & \yes & & & & \yes & \multirow{4}{*}{\makecell[l]{$\mathsf{W}^{kl}=1$, $ \text{rank}(\mathbf{D})\!\leq\!\!\frac{(b+1)(b+2)}{2}$,  \\  ${\text{rank}(\mathsf{D}^{kl})\leq \text{min}(8,b+2)}$}}  \\
    \emph{Wolf \& Shashua}~\cite{wolf2002projection} 
        &  &  & * & & & & \yes & \\
    \emph{Hartley \& Vidal}~\cite{hartley2008perspective} 
        &  &  & * & & & & \yes & \\
    \emph{Agudo~\etal}~\cite{agudo2018image} 
        &  &  & * & & & & \yes & \\[0.8pt] \hdashline \noalign{\vskip 1pt}
    \emph{Wu~\etal}~\cite{wu2020unsupervised} 
        & \yes & \yes &  & & & & & \\[0.8pt] \hdashline \noalign{\vskip 1pt}    
    Ours
        & \yes & \yes & \yes & \yes & \yes & \yes & \yes &  \\
  	\bottomrule
  \end{tabular}
  \vspace{-6pt}
  \caption{\textbf{Our formulation in relation to the existing works.}
All five priors, including the rigid prior, can be used to train our method, given the suitability of priors on the structures of interest. Indices $\{i,j,k,l\}$ iterate through all, unless mentioned otherwise. 
The asterisk(*) indicates inexact or incomplete use of a prior.  
  \label{tab:related_Work_overview}}
\end{table}
Heretofore, we have discussed the relationship of our formulation in Eq.~\eqref{eq:deep-iso-nrsfm} to rigidity, ID, and ARAP priors. 
Now, we further explore the same with respect to LR and SM, for the completeness of our theoretical framework.
To do so, let us denote $\mathsf{D}^{kl} = \mathsf{E}(\mathsf{\Lambda}^k)-\mathsf{E}(\mathsf{\Lambda}^l)$ and its stacked version ${\mathbf{D} = [(\mathsf{D}^{12})^\intercal,\ldots,(\mathsf{D}^{kl})^\intercal,\ldots,(\mathsf{D}^{(m-1)m})^\intercal]}$.
Our investigation leads to the following relationships between the formulation used in this paper and the prior of the low-rank purely on structures, without explicit camera motion.
\begin{proposition}[EDM of Low-rank structures]
For 3D structures that can be represented using $b$ linear basis, the constraints ${\text{rank}(\mathsf{D}^{kl})\leq \text{min}(8,b+2)}$ for all $k,l$ and $\text{rank}(\mathbf{D})\leq \frac{(b+1)(b+2)}{2}$ are necessary
to recover the low-rank structures, using Euclidean distance matrices.  
\end{proposition}
\begin{proof}
\vspace{-3pt}
The proof is provided in the suppl. material.\hfill \qed 
\vspace{-3pt}
\end{proof}

\begin{figure*}[!htb]
  \centering
  \vspace{-1.5em}
  \includegraphics[width=1\linewidth]{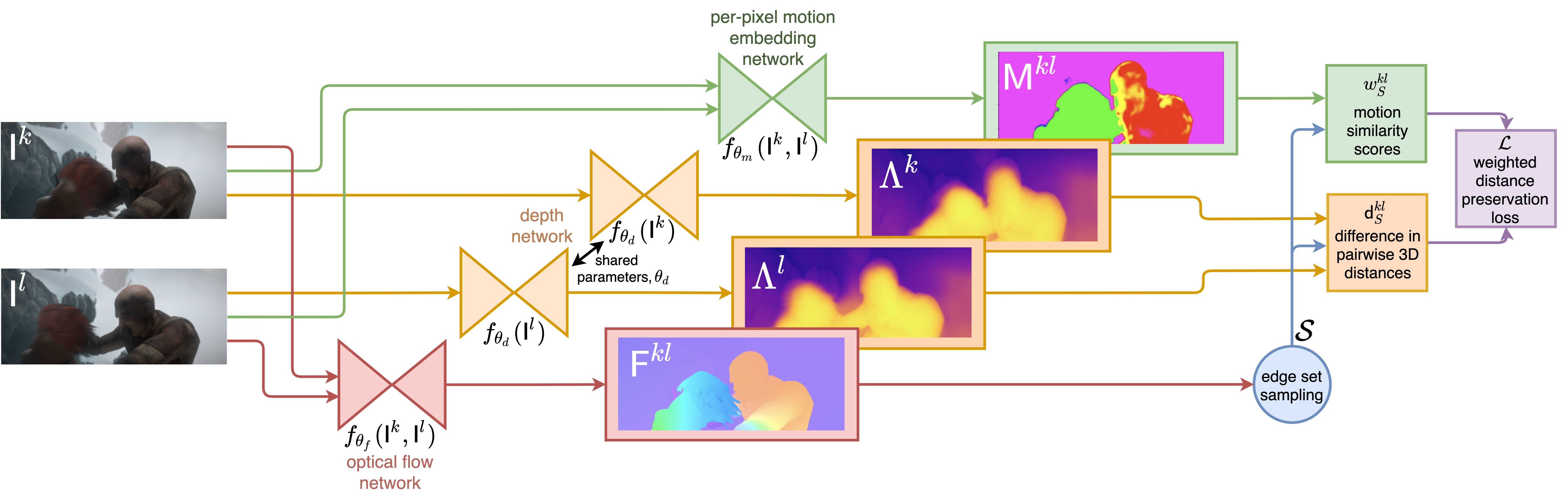}
  \vspace{-2em}
  \caption{\textbf{Unsupervised learning of depth reconstruction and motion embeddings.} Our model consists of a depth and a per-pixel motion embedding networks, which jointly learn depth map, and motion embeddings whose similarity expresses the rigidity between scene point pairs. A pre-trained flow network used for dense correspondences is included for completeness. The training objective is given in Eq.~\eqref{eq:loss-iso-nrsfm}. After training, only a single-pass through the depth network is required for the inference of depth from an image.
      \vspace{-1em}}
  \label{fig:unsupervisedEmbeddingModel-met} 
\end{figure*}

SM prior can also be integrated in our formulation by enforcing rigidity scores.
For example, rigid objects with linear constant velocity maintain highest rigidity with respect to parallel lines/planes\footnote{Which can be represented using a set of points on them.} along the velocity direction. 
In fact, the exact rigidity scores to the other points can be derived, if the velocity is also known. 
This is particularly interesting, if the motion prior and image semantics are known, \eg a rigid car driving on the rigid road with stops, or a rigid surgical tool poking an isometric organ surface. 
We represent a set of edges whose rigidity scores can be measured (or bounded) in some or all image pairs, as $\mathcal{S}^{kl}=\{s_{e}^{kl}| e\in \{e_{ij}^{kl}\}\}$. 
These scores then can also be used (if necessary with bounds) in our formulation.   
Needless to say, the SM prior can be used in conjunction with LR in practice~\cite{wolf2002projection,hartley2008perspective}. 
An overview of commonly used priors as discussed in Section~\ref{sec:revisit} in relation to previous works is given in Table~\ref{tab:related_Work_overview}, and presented in a unified framework. Our method can be used for all priors unified in Table~\ref{tab:related_Work_overview}. However, different priors are useful depending on the scene.

Having provided this unified theoretical framework, we  proceed with the experimental analysis of our proposed methodology. In this work, we are interested in reconstructing the depth of \textit{all} parts of a scene, and in that direction we choose to focus on the \textit{ARAP} prior which is sufficient to resolve the scale ambiguity between freely moving parts, in contrast to priors such as the LR in multi-body settings. That being said, the other priors we unify in Table~\ref{tab:related_Work_overview} can also be used within our framework to reconstruct scenes/objects operating under different priors, by computing rigidity scores with respect to the corresponding prior and using these scores in our formulation. 
 \section{Learning Depth Using ARAP}\label{sec:learningdepth}

We first use problem~\eqref{eq:deep-iso-nrsfm} to learn depth reconstruction from image data, given the dense correspondences between views $k$ and $l$.
We will then formalize our learning objective for a given image pair $\mathsf{I}^k$ and $\mathsf{I}^l$ of views $k$ and $l$.
Our optimization objective can be extended to multi-view images $\mathcal{I} = \{\mathsf{I}^1, \mathsf{I}^2, \cdots, \mathsf{I}^m\}$ of the same scene. 
An overview of our learning pipeline is shown in Fig.~\ref{fig:unsupervisedEmbeddingModel-met}.

\boldparagraph{Objective and overview.}
Using a deep convolutional neural network  ${\mathbf{f}_{\theta_d}(\mathcal{I}): \mathbb{R}^{H \times B \times 3} \xrightarrow{} [0,\mathbb{R}^+]^{H \times B}}$ parameterized by $\theta_d$, we wish to estimate the depth for a given RGB image of size $H \times B \times 3$, as the output of the network.

In order to train the depth network, we use the ARAP prior. For that purpose, we predict depths for the views $k$ and $l$ as: $\Lambda^k = \mathbf{f}_{\theta_d}(\mathsf{I}^k)$ and $\Lambda^l = \mathbf{f}_{\theta_d}(\mathsf{I}^l)$, respectively. In order to use the ARAP objective \eqref{eq:deep-iso-nrsfm}, we establish correspondences between views $k$ and $l$ with a pre-trained optical flow network ${\mathbf{f}_{\theta_f}(\mathcal{I} \times \mathcal{I})}$, whose weights are frozen. We use the dense correspondences for the view pair $(k,l)$ obtained from $\mathbf{f}_{\theta_f}(\mathsf{I}^k,\mathsf{I}^k)$ to compute the difference of EDMs given by $\mathsf{E}(\mathsf{\Lambda}^k)-\mathsf{E}(\mathsf{\Lambda}^l)$. Next we describe how we obtain the rigidity scores $\mathsf{W}^{kl}$ for the view pair $(k,l)$.

\boldparagraph{Motion embeddings.}
Revisiting the initial optimization problem given in~\eqref{eq:deep-iso-nrsfm}, we aim to estimate the weight matrix $\mathsf{W}^{kl}$, whose entries $w_{ij}^{kl} \in [0,1]$ represent the \textit{rigidity} between points $(X_i^k, X_j^l)$, across the transformation between the views $(k \leftrightarrow{} l)$. %
To derive this weight matrix $\mathsf{W}^{kl}$, we propose to learn \textit{per-pixel motion embeddings}, whose similarity indicates the sought rigidity between points. 
In particular, we seek for a dense and complete map of motion embeddings, given a pair of views $(k, l)$. To this end, we make use of another neural network, ${\mathbf{f}_{\theta_m}(\mathcal{I}\times \mathcal{I}): \mathbb{R}^{({H \times B \times 3})\times({H \times B \times 3})} \xrightarrow{} [0,1]^{H \times B\times v}}$ parameterized by $\theta_m$, to learn the motion embeddings:

\begin{definition}[Motion embedding matrix] A motion embedding matrix, say $\mathsf{M}^{kl}\in\mathbb{R}^{v \times n}$, is a matrix representing $n$ points in $v$-dimensional Euclidean space, in which ${\mathsf{M}^{kl}=[\mathsf{m}_1^{kl},\mathsf{m}_2^{kl},\ldots,\mathsf{m}_n^{kl}]\in\mathbb{R}^{v\times n}}$.  The entry $\mathsf{m}_i^{kl}$ represents a \textbf{motion embedding vector} for each point $X_i^k$ in the $k$-th view, representing the motion that the point undergoes across the given view pair $(k, l)$.
\label{def:motion_embeddings}
\end{definition}

\noindent
Having obtained the motion embedding matrix $\mathsf{M}^{kl}$, we derive the entries of $\mathsf{W}^{kl}$ from the pairwise distances between the motion embedding vectors $\mathsf{m}_i^{kl}$ and $\mathsf{m}_j^{kl}$. It should be noted that we expect a higher weight when the distance between the motion embedding vectors are smaller, and vice versa. With this consideration, we formulate the motion similarity scores for every pair (as illustrated in Fig.~\ref{fig:score_calculation}) as,
\begin{equation}\label{eq:motion_similarity_scores}
    w_{ij}^{kl} = 1 - tanh(\norm{\mathsf{m}_i^{kl} - \mathsf{m}_j^{kl}}).
\end{equation}

\begin{figure}[t]
  \centering
  \vspace{-10pt}
  \includegraphics[width=0.95\columnwidth,clip]{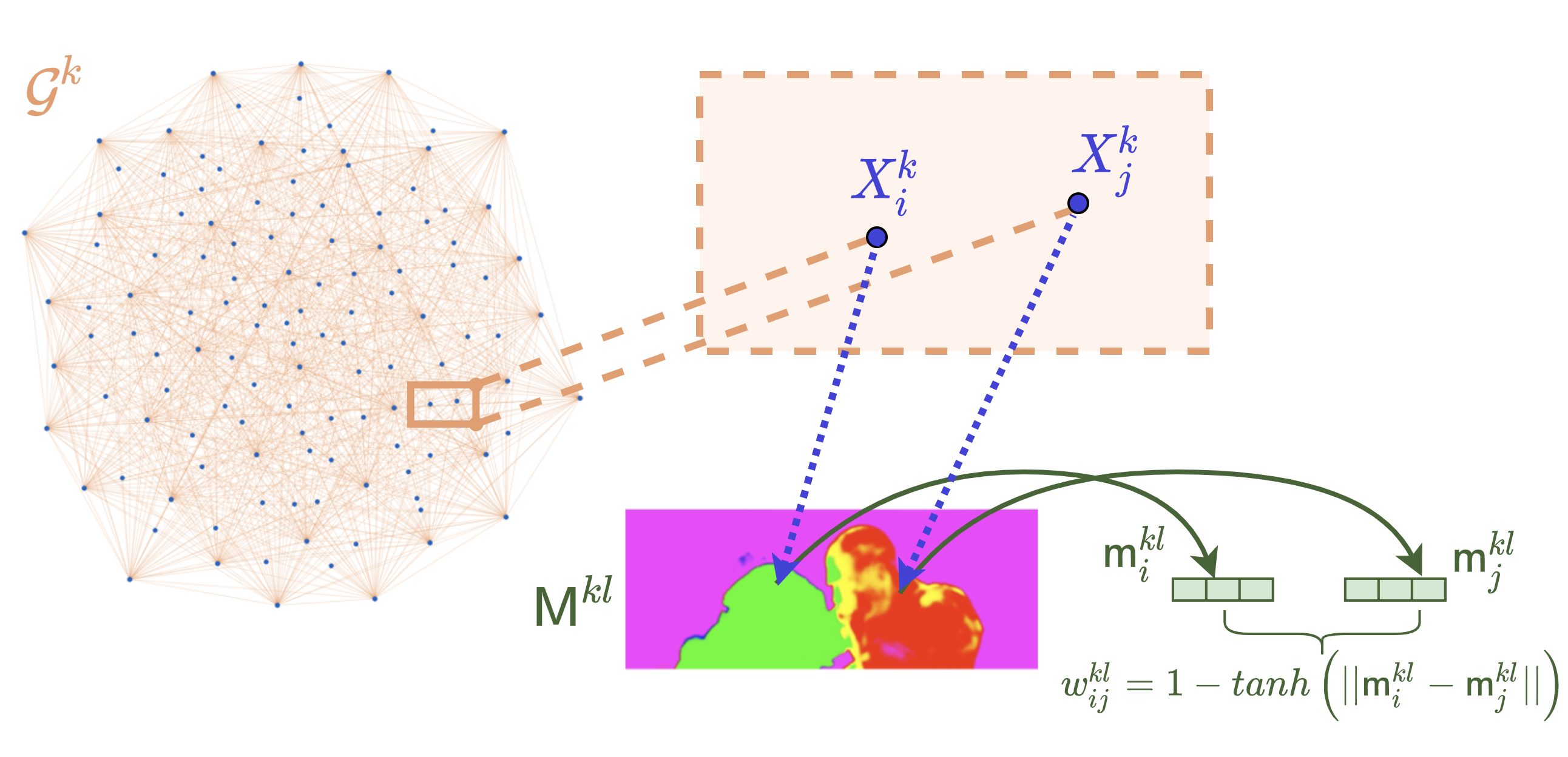}
  \vspace{-1.5em}
  \caption{\textbf{Motion similarity scores.} For a given point pair $(X_i^{k}, X_j^{k})$  from the scene graph $\mathcal{G}^k$, we retrieve the associated motion embeddings $\mathsf{m}_{i}^{kl}$ and $\mathsf{m}_{j}^{kl}$, from which we derive motion similarity scores. These scores express how much we expect the 3D distance to be preserved for each point pair.}
   \label{fig:score_calculation}
\end{figure}

\boldparagraph{Edge Sampling.}
We aim to compute the \emph{motion similarity scores}, $w_{ij}^{kl}$, ideally for all possible point pair combinations. 
However, performing this computation is not tractable considering the number of pixels.
Therefore, from the graph for the $k$-th view, say $\mathcal{G}^k = (V=\mathcal{X}^k, E=\mathcal{X}^k \times \mathcal{X}^k)$, we uniformly sample a random subset of all edges (point pairs) for the loss computation. The sampled edges are denoted as $\mathcal{S}^k$. Using this subset $\mathcal{S}^k$, we compute a weight matrix $\mathsf{W}_\mathcal{S}^{kl}$, whose entries are given by %
\begin{equation}
 w_{\mathcal{S},ij}^{kl} = \left\{
\begin{array}{ll}
      w_{ij}^{kl} & (X_{i}^k, X_{j}^{k}) \in \mathcal{S}^k, \\
      0 & otherwise. \\
\end{array} 
\right. 
 \label{eq:sparse_matrix}
\end{equation}

\boldparagraph{Loss Formulation.}
Using  ~\eqref{eq:motion_similarity_scores} and ~\eqref{eq:sparse_matrix}, we reformulate the problem of~\eqref{eq:deep-iso-nrsfm} as a loss function to train our networks, parameterized by $\theta = \{\theta_d,\theta_m\}$, as follows:
\begin{equation}
\mathcal{L}_\theta(\mathsf{\Lambda}^k, \mathsf{\Lambda}^l, \mathsf{W}_\mathcal{S}^{kl}) =
 \frac{\norm{\mathsf{W}_\mathcal{S}^{kl}\odot\big(\mathsf{E}(\mathsf{\Lambda}^k)-\mathsf{E}(\mathsf{\Lambda}^l)\big)}_{1,1}}
      {\alpha\norm{\mathsf{W}_\mathcal{S}^{kl}}_{1,1}} \;.
 \label{eq:loss-iso-nrsfm}
\end{equation}
The normalization factor $\alpha=\norm{\mathsf{E}(\mathsf{\Lambda}^k)+\mathsf{E}(\mathsf{\Lambda}^l)}_{1,1}$ is introduced for numerical stability and to avoid reconstructions with near-zero depth values.
The constraints of~\eqref{eq:deep-iso-nrsfm} are imposed in the network output: Depths $\mathsf{\Lambda}^k$ are ensured to be positive by using a sigmoid on inverse depth output from the network. Similarly, the weights $\mathsf{W}^{kl}$ are bounded by using~\eqref{eq:loss-iso-nrsfm}, followed by a $\tau$-offset and clamping of the weight values to $[0,1]$ in progression (to enforce the constraints of~\eqref{eq:deep-iso-nrsfm}). Computation of \eqref{eq:loss-iso-nrsfm} can be efficiently performed by exploiting the sparsity of $\mathsf{W}_{\mathcal{S}}^{kl}$. During the calculation, we only iterate through the indices with non-zero entries of $\mathsf{W}_{\mathcal{S}}^{kl}$. We additionally incorporated a weight-norm regularization term in our training objective, i.e. $\beta\norm{\mathsf{W}_\mathcal{S}^{kl}}_{1,1}$, to control the maximization of the weights. Further implementation details can be found in the supplementary. 
We summarize our loss computation process in Algorithm~\ref{algo:lossComputation}.

\boldparagraph{Network structure.}
Our depth network $\mathbf{f}_{\theta_d}(\mathcal{I})$ consists of a ResNet-18 based encoder and a decoder. The input to our depth network is a single RGB image, and the output from the network is a single depth map of the same size. Our per-pixel motion embedding network $\mathbf{f}_{\theta_m}(\mathcal{I}\times \mathcal{I})$ also consists of a multi-input ResNet-18 based encoder and a decoder. The input to the motion embedding network is a pair of images, whose channels are concatenated to create a single input tensor before being passed through the network. The output from the motion embedding network is a motion-embedding map, which, in our case, has 3 channels.

\begin{algorithm}[t]
\caption{\label{algo:lossComputation} $[ \mathcal{L}^{kl}_\theta]=$\textbf{computeLossARAP}$(\mathsf{I}^{k}, \mathsf{I}^{l}$)}\vspace{0.5mm}
\begin{algorithmic}[1]

   \State Sample a set of edges $\mathcal{S}^k$ with their vertices $\mathcal{V}^k$.

 \State Estimate  the motion embedding $\mathsf{M}^{kl}=\mathbf{f}_{\theta_m}(\mathsf{I}^{k}, \mathsf{I}^{l})$.
  
  \State Compute $\mathsf{W}^{kl}_\mathcal{S}$ using $\mathsf{M}^{kl}$ and~\eqref{fig:score_calculation} for $\mathcal{S}^k$.
  
  \State Establish $(i,j)$ between $(k,l)$ using  $\mathsf{F}^{kl} = \mathbf{f}_{\theta_f}(\mathsf{I}^k, \mathsf{I}^l)$.
  
  \State Start loop $s = k,l$
  \State -- Estimate the depth $\mathsf{\Lambda}^s=\mathbf{f}_{\theta_d}(\mathsf{I}^{s})$
  
 \State -- Reconstruct 3D 
 ${{\mathsf{X}^s}(\mathsf{\Lambda^s})=[\lambda_1^s\mathsf{u}_1^s,\ldots,\lambda_n^s\mathsf{u}_n^s]}$ for $\mathcal{V}^s$.
 
 \State --  Compute the EDM $\mathsf{E}(\mathsf{\Lambda}^s)$ using ${\mathsf{X}^s}(\mathsf{\Lambda^s})$.
 \State End loop
 
 \State Compute loss $\mathcal{L}_\theta(\mathsf{\Lambda}^k, \mathsf{\Lambda}^l, \mathsf{W}_\mathcal{S}^{kl})$ using \eqref{eq:loss-iso-nrsfm}.

\State Return $\mathcal{L}_\theta(\mathsf{\Lambda}^k, \mathsf{\Lambda}^l, \mathsf{W}_\mathcal{S}^{kl})$.
\vspace{-6pt}
\end{algorithmic}
\algrule
Loss computation between two views for ARAP prior.
\end{algorithm}

\section{Experimental Results}
We perform a variety of experiments to demonstrate the performance of our unsupervised monocular pipeline for non-rigid scenes. In addition to assessing the depth reconstruction performance, we also evaluate the learned motion embeddings in terms of the motion segmentation task.

\begin{table*}[tb]
\vspace{-2mm}
  \centering
  \scriptsize
  \setlength{\tabcolsep}{1.5pt}
  \renewcommand{\arraystretch}{1.2}
  \begin{tabular}{ll|ccccccc|ccccccc}
    \thickhline %
    & & \multicolumn{7}{c|}{VolumeDeform \cite{volumedeform} dataset} &
        \multicolumn{7}{c}{MPI Sintel \cite{sintel} dataset} \\
    \textbf{Scene} &\textbf{Method}                   & \cellcolor{red!25}Abs Rel $\!\downarrow$&  \cellcolor{red!25}Sq Rel $\!\downarrow$&  \cellcolor{red!25}RMSE $\!\downarrow$& \cellcolor{red!25}RMSE$_{log}$ $\!\downarrow$ & 
    \cellcolor{blue!25}\textbf{\tiny$\delta\!\!<\!\!1.25$}   $\!\uparrow$& \cellcolor{blue!25}\textbf{\tiny$\delta\!\!<\!\!1.25^2$} $\!\uparrow$& \cellcolor{blue!25}\textbf{\tiny$\delta\!\!<\!\!1.25^3$} $\!\uparrow$
    & \cellcolor{red!25}Abs Rel $\!\downarrow$ &  \cellcolor{red!25}Sq Rel $\!\downarrow$ &  \cellcolor{red!25}RMSE $\!\downarrow$ & \cellcolor{red!25}RMSE$_{log}$ $\!\downarrow$ & \cellcolor{blue!25}\textbf{\tiny$\delta\!\!<\!\!1.25$}   $\!\uparrow$& \cellcolor{blue!25}\textbf{\tiny$\delta\!\!<\!\!1.25^2$} $\!\uparrow$& \cellcolor{blue!25}\textbf{\tiny$\delta\!\!<\!\!1.25^3$} $\!\uparrow$\\ \thickhline
    \textbf{Mostly}    & PackNet \cite{packnet}  &0.691  &2.879   &10.544   &1.739   &\textbf{0.425}  &\textbf{0.591} &0.667 
    
    &0.355 & 9.278 & 10.923 & 0.544 & \underline{0.499} & 0.789 & \underline{0.885}\\
    
        \textbf{Rigid}    &Li~\etal \cite{google-depth}  &0.871 &5.267 &\textbf{4.284} &0.757 &0.214 &0.496 &0.680
    &0.412 &3.798 & \textbf{3.256} & 0.663 & 0.476 & 0.676 & 0.817 \\
    
    \textbf{Seq.}     & Ours w/ motion &\underline{0.562} & \underline{2.483} &5.327  &\underline{0.700} &0.294 &0.523 &\underline{0.688} 
    
    &\textbf{0.213} & \textbf{1.400} & 6.724 & \underline{0.533} & \textbf{0.566} & \textbf{0.807} & \textbf{0.886} \\
    
       & Ours w/o motion &\textbf{0.549} &\textbf{2.441} &\underline{5.315} &\textbf{0.694} &\underline{0.300} &\underline{0.525} &\textbf{0.704}
         &\underline{0.255} & \underline{1.710} & \underline{3.780} & \textbf{0.528} & 0.493 & \underline{0.794} & 0.883
    \\ \hline

    \textbf{Non-}  & PackNet \cite{packnet} &0.774 &3.595 &4.196 &2.680 &0.242 &0.386 &0.535
   &1.480 & 14.737 & 7.465 & 0.882 & 0.327 & 0.508 & 0.640 \\  %
     \textbf{Rigid}    &Li~\etal \cite{google-depth}  &1.217 &9.254 &4.391 &\underline{0.852} &0.211 &0.419 &0.594
    &1.312 & 15.242 & \textbf{6.278} & 0.895 & 0.320 & 0.498 & 0.618 \\
    
    \textbf{Seq.}  & Ours w/ motion &\underline{0.562} &\underline{1.675} &\textbf{2.576} &\textbf{0.533} &\textbf{0.492} &\textbf{0.673} &\textbf{0.823} 
    &\underline{0.638} & \underline{4.833} & \underline{7.602} & \textbf{0.833} & \textbf{0.442} & \textbf{0.615} & \textbf{0.713}\\
       & Ours w/o motion &\textbf{0.519} &\textbf{1.363} &\underline{2.767} &0.553 &\underline{0.373} &\underline{0.635} &\underline{0.817}
    &\textbf{0.617} & \textbf{4.686}  & 7.681 & \underline{0.835} & \underline{0.406} & \underline{0.594} & \underline{0.708}\\  \hline

    \textbf{All}          & PackNet \cite{packnet} &0.673 & 3.623 & 6.585 & 2.227 & 0.299 & 0.465 & 0.576
    &1.134 & 13.057 & 8.529 & \underline{0.778} & 0.380 & 0.594 & 0.716\\
     \textbf{Seq.}    &Li~\etal \cite{google-depth}  &0.992 & 6.855 &\textbf{4.268} & 0.777 & 0.211 & 0.465 & 0.652
    &1.012 & 11.427 & \textbf{5.271} & 0.818 & 0.372 & 0.557 & 0.684 \\
      & Ours w/ motion &\underline{0.578} & \underline{2.513} & 5.210 & \underline{0.726} & \textbf{0.400} & \textbf{0.597} & \underline{0.748}
    &\textbf{0.507} & \textbf{3.777} & 7.332 & \textbf{0.741} & \textbf{0.480} & \textbf{0.674}    & \textbf{0.766}\\
          & Ours w/o motion&\textbf{0.540} & \textbf{2.301} & \underline{4.294} &\textbf{0.721} & \underline{0.349} & \underline{0.582} & \textbf{0.753}
          &\underline{0.505} & \underline{3.770} & \underline{6.481} & \textbf{0.741} & \underline{0.433} & \underline{0.655} & \underline{0.762}
    \\ \thickhline %
  \end{tabular}
  \vspace{-8pt}
  \caption {\textbf{Depth reconstruction evaluation for  VolumeDeform \cite{volumedeform} and MPI Sintel \cite{sintel} datasets.} Sequences are classified into two categories: mostly rigid or non-rigid. For MPI Sintel, scenes with less than 10\% dynamic pixels are labeled as mostly-rigid. VolumeDeform split is based on our qualitative analysis. Our method with motion embeddings results in superior performance in highly non-rigid scenes in both datasets. Best results are in bold, second best are underlined. Both of our methods generally perform better than our baselines, and for the non-rigid sequences, we can see the benefits of using our method with motion embeddings.
  \vspace{-5mm}}
\label{tab:depthtable}
\end{table*}

\boldparagraph{Datasets.}
In our experiments, we use \textit{MPI Sintel}~\cite{sintel}, \textit{VolumeDeform}~\cite{volumedeform}, and \textit{Hamlyn Laparascopic Video Dataset}~\cite{hamlyn}.
From the \textit{MPI Sintel} training subset, we select a set of 14 final-pass sequences with varying levels of motions. We also use the \textit{VolumeDeform} dataset consisting of 8 sequences for further experiments on deforming scenes. Lastly, we evaluate our method on the \textit{Hamlyn Centre Laparoscopic Video Dataset}, which consists of rectified stereo image pairs collected from a partial nephrectomy. Ground-truth depth maps for \textit{VolumeDeform} are obtained from the provided depth recordings; and OpenSFM \cite{opensfm-1, opensfm-2} was used to obtain ground-truth depth for \textit{Hamlyn} using the calibrated stereo pairs in the dataset. For \textit{Sintel}, both ground-truth depth maps as well as optical flow maps are provided, which is not the case for the other two datasets. A pre-trained supervised optical flow network, RAFT~\cite{raft}, performs reasonably well on \textit{VolumeDeform} for estimating dense correspondences. On \textit{Hamlyn}, due to a large domain gap, we fine-tune an unsupervised model, DDFlow~\cite{ddflow}.

\begin{figure}[t]
  \centering
  \scriptsize
  \setlength{\tabcolsep}{1pt}
  \newcommand{\sz}{0.31}           %
  \newcommand{\cgs}{\hspace{3pt}}   %
  \begin{tabular}{cccc}
    \multirow{1}{*}[73pt]{\rotatebox{90}{\textbf{Motion\hspace{8pt}Depth\hspace{9pt}Image}}} & 
    \includegraphics[width=\sz\columnwidth]{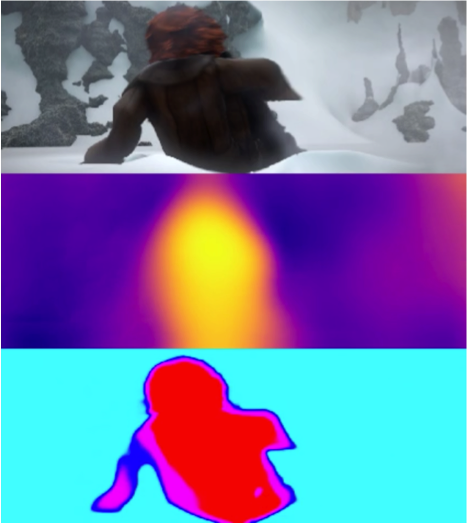} &
    \includegraphics[width=\sz\columnwidth]{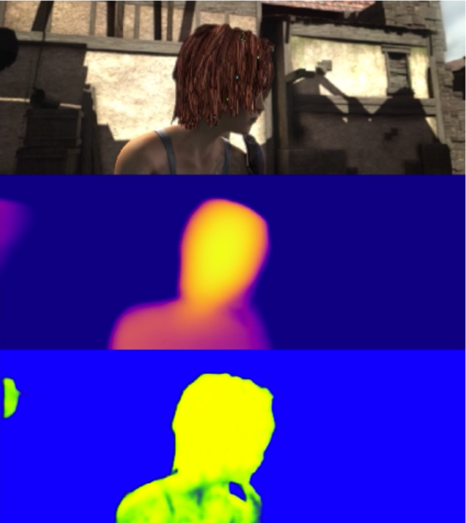} &
    \includegraphics[width=\sz\columnwidth]{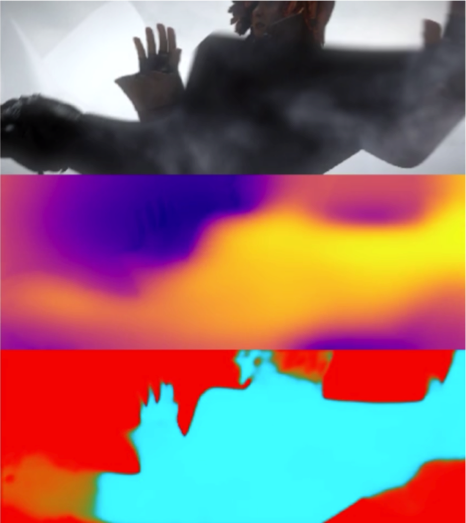} 
    \\
    \multirow{1}{*}[73pt]{\rotatebox{90}{\textbf{Motion\hspace{8pt}Depth\hspace{9pt}Image}}} & 
    \includegraphics[width=\sz\columnwidth]{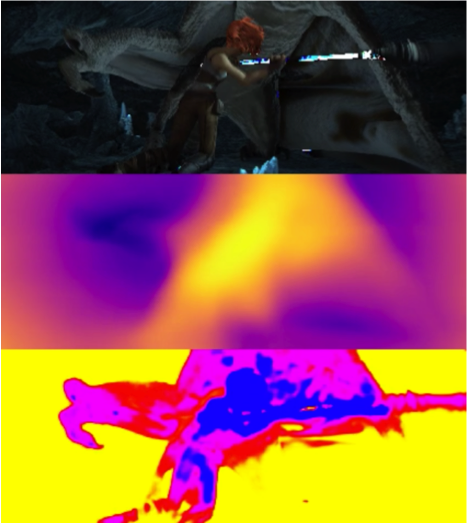} &
    \includegraphics[width=\sz\columnwidth]{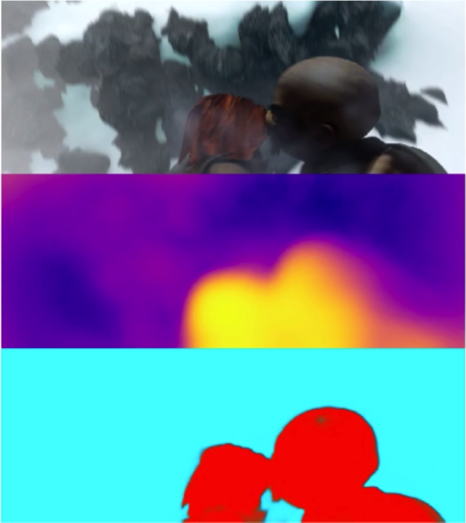} &
    \includegraphics[width=\sz\columnwidth]{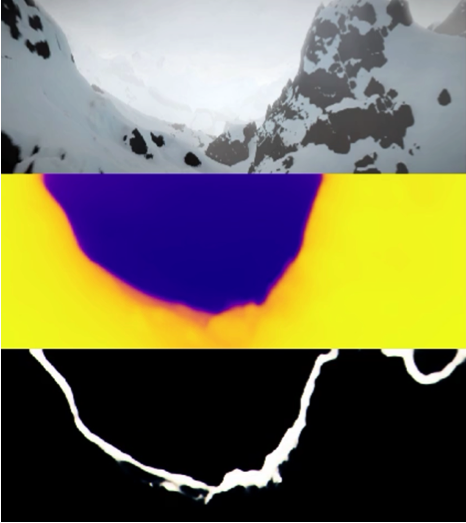} 
    \\
    \multirow{1}{*}[73pt]{\rotatebox{90}{\textbf{Motion\hspace{8pt}Depth\hspace{9pt}Image}}} & 
    \includegraphics[width=\sz\columnwidth]{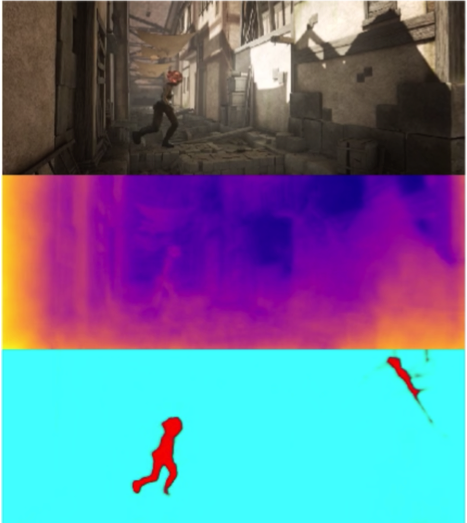} &
    \includegraphics[width=\sz\columnwidth]{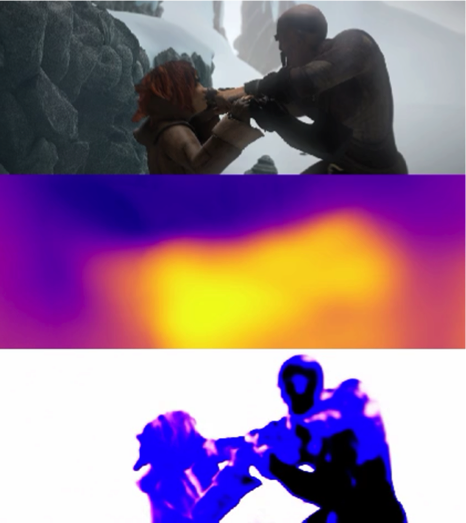} &
    \includegraphics[width=\sz\columnwidth]{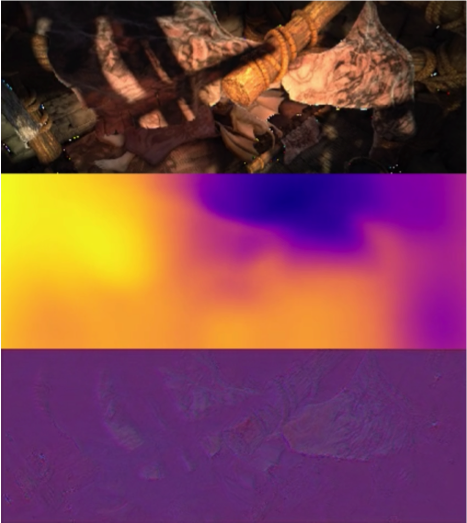}\\[-8pt]
  \end{tabular}
\caption{\textbf{Our qualitative results with depth and motion networks for  MPI Sintel dataset.} Top to bottom: RGB, predicted depth, and predicted embeddings. Embedding colors encoding information about motion clusters are not related across sequences. \vspace{-4mm}}
    \label{fig:unsupervisedEmbedding-sintel}

\end{figure}

\boldparagraph{Training Details.} We follow two different sets of experiments: test-time training or inference. As the small volumes of \textit{MPI Sintel} and \textit{VolumeDeform} datasets do not yet enable us to train a depth network that can generalize, we perform \textit{test-time training} for each image sequence, \ie we train separate models for each sequence, for our methods and the deep baselines. In this pipeline, for all of the methods, the training is rather used as an optimization procedure for solving for the depth in all frames, given an image sequence. For the \textit{Hamlyn} dataset where the amount of data enables us to train models that could generalize, we present inference results. We use the original training and test split from the Hamlyn dataset, and further create train/val/test splits. We implemented our model in PyTorch \cite{pytorch}. 
In all experiments, the ResNet18-based depth encoder and motion-embedding encoder were initialized with weights from ImageNet \cite{imagenet} pre-training. Adam \cite{adam} optimizer with  $\beta_1= 0.9$,  $\beta_2=0.999$ was used, in combination with a learning rate decay by 0.1 every 10 epochs. We follow a two-stage training. In the first stage, we jointly train the motion-embedding network and the depth network. In the second stage, we freeze the weights of the motion-embedding network, and re-initialize the depth network training. In the latter stage, we apply a $\tau$-offset to the weights, to impose the constraint from~\eqref{eq:deep-iso-nrsfm}. For the selection of the edges, we uniformly sample pairs of points with correspondences in the consecutive frame. We empirically find 100K pairs (edges) to provide a suitable trade-off between increased memory requirements and slow convergence. 
Further details are in the supplementary material.

\boldparagraph{Depth reconstruction results.}
Our qualitative results are demonstrated in Fig.~\ref{fig:teaser}, ~\ref{fig:unsupervisedEmbedding-sintel} and  ~\ref{fig:unsupervisedEmbedding-hamlyn}.
We report our results by performing per-image median ground-truth scaling for each method, as introduced in \cite{zhou}.
For the \textit{MPI Sintel} dataset, we evaluate the performance only where the GT depth is smaller than 50 m. For the other datasets, we evaluate the depth for every point where the GT is available. As shown in  Table~\ref{tab:depthtable}, \ref{tab:depthtable_segmented} and \ref{tab:hamlyntable}, we report results from ``Ours w/ motion", and ``Ours w/o motion". ``Ours w/ motion" refers to the setting we described in Section~\ref{sec:learningdepth}, where we learn the motion similarity scores serving the ARAP prior. ``Ours w/o motion" refers to the same unsupervised pipeline, except that all rigidity scores are explicitly set to 1. With this modification, piecewise-rigidity is implied during training. In Table \ref{tab:depthtable_segmented}, we further compare the performance of our methods, by separately evaluating the performance in dynamic and static parts of the images. In Table~\ref{tab:hamlyntable}, the results from the non-rigid reconstruction models \textit{DLH} \cite{dai2012simple}, \textit{MDH} \cite{probst2018incremental} and \textit{MaxRig} \cite{maxrig} were obtained using the complete test sequence to perform reconstruction. The evaluations for \textit{DLH} and \textit{MaxRig} were performed only at the reconstructed points as these methods provide sparse reconstructions.

\begin{table}[tb]
  \centering
  \scriptsize
    \setlength{\tabcolsep}{1.9pt}
    \renewcommand{\arraystretch}{1.3}
    \begin{tabular}{ll|ccccccc}
    \thickhline %
     & 
    &\cellcolor{red!25}Abs & 
    \cellcolor{red!25}Sq & 
    \cellcolor{red!25}RMSE & 
    \cellcolor{red!25}RMSE$_\text{log}$ & 
    \cellcolor{blue!25}\textbf{\tiny$\delta\!\!<\!\!1.25$} & 
    \cellcolor{blue!25}\textbf{\tiny$\delta\!\!<\!\!1.25^2$} & 
    \cellcolor{blue!25}\textbf{\tiny$\delta\!\!<\!\!1.25^3$} \\
    \textbf{Scene}  & \textbf{Method}         & 
    \cellcolor{red!25}Rel$\downarrow$ & 
    \cellcolor{red!25}Rel$\downarrow$&  
    \cellcolor{red!25}$\downarrow$ & 
    \cellcolor{red!25}$\downarrow$ & 
    \cellcolor{blue!25}$\uparrow$ & 
    \cellcolor{blue!25}$\uparrow$ & 
    \cellcolor{blue!25}$\uparrow$\\ \thickhline

    \textbf{Rigid}  & w/ motion  & \textbf{0.387} &\textbf{2.276} &5.954	&\textbf{0.778}	&0.407	&0.632	&0.718 \\   
      
     \textbf{Parts}  &w/o motion &0.393	&2.335	&\textbf{5.529}	&0.787	&\textbf{0.431}	&\textbf{0.636} &\textbf{0.725} \\  \hline 

       \textbf{Non-Rig}     & w/ motion &\textbf{0.702}	&\textbf{1.828}	&\textbf{1.618}	&\textbf{0.503}	&\textbf{0.348}	&\textbf{0.608}	&\textbf{0.795} \\
    
     \textbf{Parts}  &w/o motion & 0.889	&5.879	&8.919	&0.551	&0.329	&0.592	&0.732 \\
     \thickhline %
    
  \end{tabular}
  \vspace{-8pt}
  \caption {\textbf{Metrics computed over the dynamic parts or static parts separately.} The separation between the dynamic and static parts was obtained using the ground-truth motion segmentation maps. Metrics are averaged over all images from all sequences.}
\label{tab:depthtable_segmented}
\end{table}

\begin{figure}[t]
    \vspace{-2mm}
  \centering
  \small
  \setlength{\tabcolsep}{1pt}
  \newcommand{\sz}{0.31}           %
  \newcommand{\cgs}{\hspace{3pt}}   %
  \begin{tabular}{cccc}
    \multirow{1}{*}[60pt]{\rotatebox{90}{\textbf{\;Depth\;\;\;\;\;Image\;}}} & 
    \includegraphics[width=\sz\columnwidth]{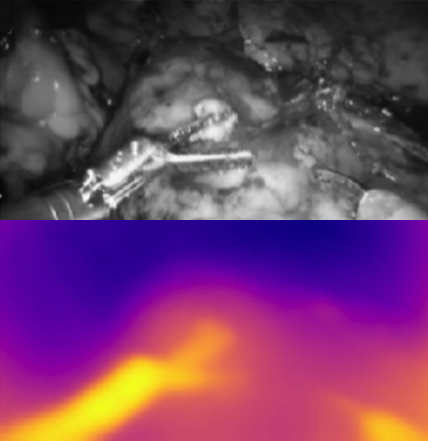} &
    \includegraphics[width=\sz\columnwidth]{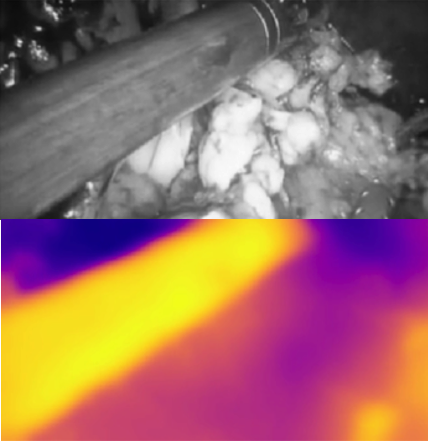} &
    \includegraphics[width=\sz\columnwidth]{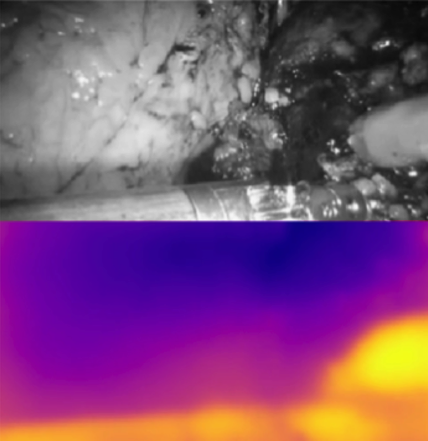} \\[-8pt]
  \end{tabular}
    \caption{\textbf{Depth reconstruction results for the Hamlyn dataset.} Images and depth predictions are presented. The  images shown in grayscale here are used in RGB format during training.}
    \label{fig:unsupervisedEmbedding-hamlyn}
\end{figure}

\begin{table}[tb]
\vspace{-2mm}
\centering
\scriptsize
\setlength{\tabcolsep}{2.5pt}
\renewcommand{\arraystretch}{1.3}
\begin{tabular}{l|ccccccc}
\thickhline %
 & 
\cellcolor{red!25}Abs & 
\cellcolor{red!25}Sq & 
\cellcolor{red!25}RMSE & 
\cellcolor{red!25}RMSE$_\text{log}$ & 
\cellcolor{blue!25}\textbf{\tiny$\delta\!\!<\!\!1.25$} & 
\cellcolor{blue!25}\textbf{\tiny$\delta\!\!<\!\!1.25^2$} & 
\cellcolor{blue!25}\textbf{\tiny$\delta\!\!<\!\!1.25^3$} \\
\textbf{Method}                   & 
\cellcolor{red!25}Rel$\downarrow$ & 
\cellcolor{red!25}Rel$\downarrow$&  
\cellcolor{red!25}$\downarrow$ & 
\cellcolor{red!25}$\downarrow$ & 
\cellcolor{blue!25}$\uparrow$ & 
\cellcolor{blue!25}$\uparrow$ & 
\cellcolor{blue!25}$\uparrow$\\ \thickhline

PackNet \cite{packnet}          &0.389 &2.533 &4.378 &0.418 &0.469 &0.733 &0.882    \\ 
Li~\etal \cite{google-depth}    &0.344 &2.078 &4.169 &0.385 &0.484 &0.772 &0.907     \\ 
DLH \cite{dai2012simple}                  &0.644 &5.150 &6.108 &0.611 &0.291 &0.547 &0.744     \\ 
MDH \cite{probst2018incremental}&1.222 &14.210 &4.530 &0.412 &0.572 &0.829 &0.945    \\ 
MaxRig \cite{maxrig}            &0.232 &1.083 &3.247 &0.294 &0.587 &0.850 &0.962     \\ 
N-NRSfM \cite{Sidhu2020} &0.392 & 2.548 &5.295 &0.730 &0.361 &0.626 &0.779\\ 
Ours w/ motion                   &\underline{0.217} &\underline{0.941} &\underline{3.120} &\underline{0.279} &\underline{0.608} &\underline{0.886} &\underline{0.977}    \\ 
Ours w/o motion                 &\textbf{0.213} &\textbf{0.921} &\textbf{3.065} &\textbf{0.272} &\textbf{0.618} &\textbf{0.891} &\textbf{0.980} 
\\ \hline 
\end{tabular}
\vspace{-0.8em}
\caption {\textbf{Depth reconstruction performance evaluation for Hamlyn Dataset \cite{hamlyn}.} Both of our methods result in superior performance compared to previous methods.}
\label{tab:hamlyntable}
\end{table}

\vspace{-1mm}

\boldparagraph{Motion segmentation results.}
We also evaluate the performance of our per-pixel motion embedding. As the \textit{ground-truth} motion-embeddings are not available, we perform the evaluation on \textit{MPI Sintel} dataset using the available ground-truth motion segmentation maps \cite{gt-sintel-motsegm}. With that purpose, we first calculate the static background embedding vector by taking channel-wise median of the image border embeddings, assuming the image border mostly consists of the static background. Then we separate the static part from the dynamic parts of the scene, by thresholding the distances between the static background embedding and each per-pixel motion embedding vectors. In Table~\ref{tab:threshold2}, we compare the performance of our model with two recent works \cite{motsegm-baseline} and \cite{loquercio}, and demonstrate our competitive results for the motion segmentation task in Fig.~\ref{fig:mot-segm-results}.

\begin{table}[tb]
\vspace{-2mm}
  \centering
  \footnotesize
  \setlength{\tabcolsep}{1.6pt}
  \renewcommand{\arraystretch}{1.1}
  \floatbox[\capbeside]{table}[0.45\columnwidth]%
  {\caption {\textbf{Average overall pixel accuracy on MPI Sintel~\cite{sintel}.} Overall pixel accuracy and intersection-over-union metrics were averaged over the sequences.}
  \label{tab:threshold2}}%
  {\begin{tabular}{l|cc}
    \thickhline
           & \cellcolor{blue!25} Mean           &\cellcolor{blue!25} Mean \\
    \textbf{Method} & \cellcolor{blue!25} ACC $\uparrow$ &\cellcolor{blue!25} IoU $\uparrow$ \\\thickhline
    Taniai~\etal \cite{motsegm-baseline} & \underline{0.872}  & \underline{0.728}\\
    Yang~\etal~\cite{loquercio}          & 0.755  & 0.519\\
    Ours w/ motion      & \textbf{0.912} & \textbf{0.731}\\
    \thickhline 
  \end{tabular}}
\end{table}

\begin{figure}[t]
\vspace{-3mm}
  \centering
  \scriptsize
  \setlength{\tabcolsep}{1pt}
  \newcommand{\sz}{0.314}           %
  \newcommand{\cgs}{\hspace{3pt}}   %
  \begin{tabular}{cccc}
    \multirow{1}{*}[18pt]{\rotatebox{90}{\textbf{Frame}}} & 
    \includegraphics[width=\sz\columnwidth]{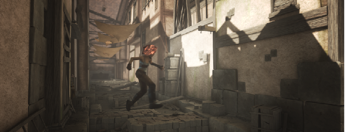} &
    \includegraphics[width=\sz\columnwidth]{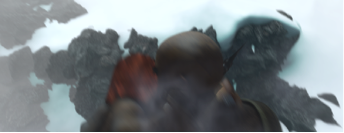} &
    \includegraphics[width=\sz\columnwidth]{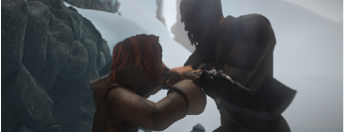} 
    \\
    
    \multirow{1}{*}[15pt]{\rotatebox{90}{\textbf{GT}}}  &   
    \includegraphics[width=\sz\columnwidth]{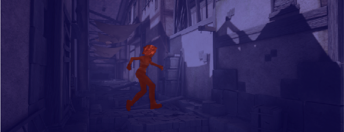} &
    \includegraphics[width=\sz\columnwidth]{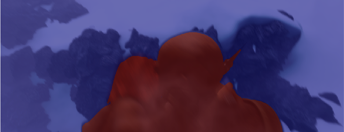} &
    \includegraphics[width=\sz\columnwidth]{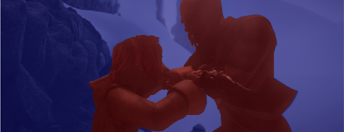} 
    \\
    
    \multirow{1}{*}[15pt]{\rotatebox{90}{\cite{motsegm-baseline}}} &   
    \includegraphics[width=\sz\columnwidth]{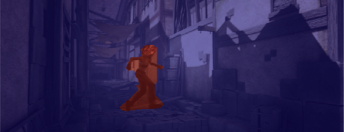} &
    \includegraphics[width=\sz\columnwidth]{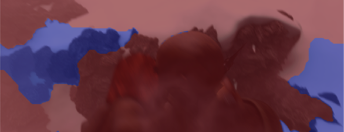} &
    \includegraphics[width=\sz\columnwidth]{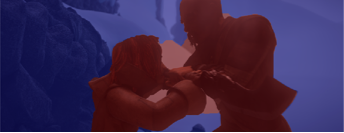} 
    \\
    
    \multirow{1}{*}[15pt]{\rotatebox{90}{\cite{loquercio}}}  &   
    \includegraphics[width=\sz\columnwidth]{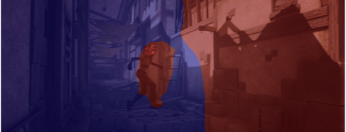} &
    \includegraphics[width=\sz\columnwidth]{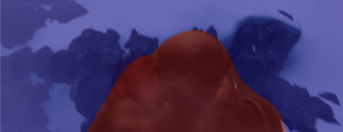} &
    \includegraphics[width=\sz\columnwidth]{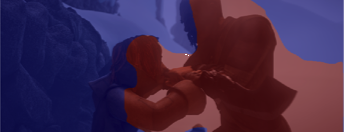} 
    \\

    \multirow{1}{*}[15pt]{\rotatebox{90}{\textbf{Ours}}}  &   
    \includegraphics[width=\sz\columnwidth]{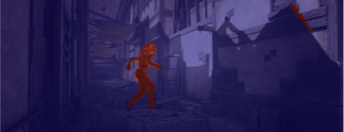} &
    \includegraphics[width=\sz\columnwidth]{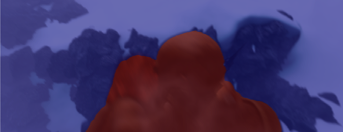} &
    \includegraphics[width=\sz\columnwidth]{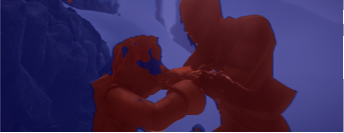} \\[-8pt]
  \end{tabular}
 \caption{\textbf{Motion segmentation performance.} Our method gives competitive results despite motion segmentation being only an auxiliary output in our pipeline.}
    \label{fig:mot-segm-results}
\end{figure}

\vspace{1mm}
\boldparagraph{Discussion.}
Through numerous experiments, we demonstrated that our formulation with ARAP assumption performs well, even for highly dynamic scenes. However, for some scenes with small deformations, we have observed that the pairwise rigidity assumption performs better. Such behaviour is often due to the short video sequence-based supervision.
Often, the scene rigidity may be maintained for a small duration. More importantly, large scene parts are often rigid, which offer better reconstruction when stronger priors are employed. However, rigid reconstruction can be performed rather easily, and despite the fact that most of the time dynamic and deformable objects are of higher interest, they are also more difficult to reconstruct without a suitable prior. In this direction, we investigated several of the commonly used priors, and demonstrated the utility of ARAP assumption for our pipeline. We believe that our framework can be very useful when the scene-specific priors are known for individual use-cases. Note that our formulation highly differs from existing unsupervised monocular depth methods. More specifically, the sole source of self-supervision in our case is the minimization of the geometric errors, and not the photometric errors. This is particularly beneficial for avoiding reconstruction errors that could stem from photometric errors, especially in the presence of deformations. This fact, along with the motion-independent formulation may explain why our pair-wise rigidity works better than a motion-explicit rigid unsupervised learning method \cite{packnet} for non-rigid scenes.
Even if we seek for pair-wise rigidity in non-rigid scenes, since the exact rigid reconstruction is not feasible, the optimal reconstruction is bound to be ARAP. 
\vspace{-2.5mm}

\section{Conclusion}
\vspace{-2mm}
We demonstrated the possibility for unsupervised learning of depth using monocular videos of non-rigid scenes in order to infer depth from a single image. Unsupervised learning is enabled by established priors used in the NRSfM literature. 
Building upon existing works, we reformulated commonly used NRSfM priors within a unified framework suitable for neural network training. We further investigated the utility of the ARAP prior for dense depth reconstruction from a single view. Our experiments demonstrate improvements over the alternative methods on numerous datasets. The improvements can be attributed to our geometric error-based loss function~\eqref{eq:loss-iso-nrsfm} computed directly in the 3D reconstruction space. We assumed only the minimum necessary prior in this work, as we aimed to learn without any supervision (including additional scene priors). Depending on scene properties, different priors can be integrated into our pipeline in a similar way, boosting the performance further.\vspace{1mm}

\noindent
\begin{minipage}{\columnwidth}
	\vspace{3pt}
	\footnotesize
	\noindent
	\textbf{Acknowledgments.}~
	This research was partly supported by Innosuisse funding (Grant No.~34475.1 IP-ICT), by FIFA, and by the ETH Zurich project with Specta and EU H2020 ENCORE (Grant No.~820434).
\end{minipage}

{
\small
\bibliographystyle{ieee_fullname}
\bibliography{00_egbib}
}

\newpage
\clearpage

\twocolumn[
\begin{center}
  {\Large \bf \Large{\bf Unsupervised Monocular Depth Reconstruction of Non-Rigid Scenes} \\ -- Appendix -- \par}
  \vspace*{12pt}
\end{center}
]

\appendix

\setcounter{section}{0}
\setcounter{theorem}{0}
\section{Proof of Proposition on Euclidean Distance Matrices (EDM) of Low-Rank Structures}  \label{sec:proof}
In this section we are providing the proof of the proposition on Euclidean distance matrices of low-rank structures, for the completeness of our framework which unifies different priors from the non-rigid reconstruction literature, using EDM measures across views.

\setcounter{section}{4}
\setcounter{theorem}{1}
\begin{proposition}[EDM of Low-rank structures]
For 3D structures that can be represented using $b$ linear bases, the constraints ${\text{rank}(\mathsf{D}^{kl})\leq \text{min}(8,b+2)}$ for all $k,l$ and $\text{rank}(\mathbf{D})\leq \frac{(b+1)(b+2)}{2}$ are necessary
to recover the low-rank structures, using Euclidean distance matrices.  
\end{proposition}

\setcounter{section}{1}
\setcounter{equation}{0}
\begin{proof}
Let us denote the low-rank factorization of the 3D structure $\mathsf{X}^k $ in view k as $\mathsf{X}^k = \Mmat_k \mathsf{B}$, with shape basis $\mathsf{B} \in \mathbb{R}^{b \times n}$, and view-specific projection $\Mmat_k \in \mathbb{R}^{3 \times b}$~\cite{wolf2002projection}.

\vspace{1em}
\noindent\textbf{(i) Rank of $\mathsf{D}^{kl}$.} We begin the first part of the proposition. By using the EDM in Eq.~(1), we can write  $\mathsf{D}^{kl} = \mathsf{E}(\mathsf{X}^k)-\mathsf{E}(\mathsf{X}^l)$ as follows,
\begin{align}
\label{eq:rank_Dkl}
\begin{split}
    \mathsf{D}^{kl}  = & \,
    \underbrace{\text{diag}(\mathsf{G}(\mathsf{X}^k)-\mathsf{G}(\mathsf{X}^l))\mathbf{1}^\intercal}_{A_{k,l}} \\
    & \underbrace{-2(\mathsf{G}(\mathsf{X}^k)-\mathsf{G}(\mathsf{X}^l))}_{B_{k,l}}\\
    &+
    \underbrace{\mathbf{1}\text{diag}(\mathsf{G}(\mathsf{X}^k)-\mathsf{G}(\mathsf{X}^l))^\intercal}_{C_{k,l}}, \\
    \text{with} \quad \mathsf{G}(\mathsf{X}^{k,l}) =&  \, \mathsf{B}^\intercal \Mmat_{k,l}^\intercal \Mmat_{k,l} \mathsf{B}.\\
\end{split}
\end{align}
Using $\text{rank}(A_{k,l}) \leq 1$ and $\text{rank}(C_{k,l}) \leq 1$, together with, 
\begin{align}
\begin{split}
    \text{rank}(B_{k,l}) \leq \, & \text{rank}(\mathsf{B}^\intercal \Mmat_{k}^\intercal \Mmat_{k}
    \mathsf{B}) +
    \text{rank}(\mathsf{B}^\intercal \Mmat_{l}^\intercal \Mmat_{l}
    \mathsf{B})\\
    \leq \, & \text{min}(6,b),\\
\end{split}
\end{align}
we obtain,
\begin{align}
\begin{split}
\text{rank}(\mathsf{D}^{kl}) &\leq \text{rank}(A_{k,l}) + \text{rank}(B_{k,l})+  \text{rank}(C_{k,l}) \\
 &\leq  \text{min}(8,b+2).
\end{split}
\end{align}

\vspace{1em}
\noindent\textbf{(ii) Rank of $\mathbf{D}$.} Regarding the second part of the proposition, we use $\mathbf{D}^\intercal = [(\mathsf{D}^{12})^\intercal,\ldots,(\mathsf{D}^{kl})^\intercal,\ldots,(\mathsf{D}^{(m-1)m})^\intercal]$ to denote differences of EDMs between m pairs of views.
Using Eq.~\eqref{eq:rank_Dkl}, we can expand $\mathbf{D}$ as follows,
\begin{align}
\begin{split}
\mathbf{D} =\underbrace{\begin{bmatrix}A_{1,2} \\ \vdots \\ A_{m-1,m}\\ \end{bmatrix}}_{\mathbf{A}} +
\underbrace{\begin{bmatrix}B_{1,2} \\ \vdots \\ B_{m-1,m}\\ \end{bmatrix}}_{\mathbf{B}} + 
\underbrace{\begin{bmatrix}C_{1,2} \\ \vdots \\ C_{m-1,m}\\ \end{bmatrix}}_{\mathbf{C}},
\end{split}
\end{align}
and analyse the three components separately.

\noindent$\mathbf{A:}$ Matrices $A_{k,l}$ are formed by $n$ times repeating the column vector $\text{diag}(\mathsf{B}^\intercal (\Mmat_{k}^\intercal \Mmat_{k} -\Mmat_{l}^\intercal \Mmat_{l} )\mathsf{B})$. Therefore, vertical stacking $A_{k,l}$ to obtain $\mathbf{A}=[A_{1,2}^\intercal \dots A_{m-1,m}^\intercal ]$ will also result in a matrix of at most rank 1, \ie $\text{rank}(\mathbf{A}) \leq 1$.

\noindent$\mathbf{B:}$ Assuming the rank of matrices $B_{k,l}$ is bounded by $\text{rank}(B_{k,l}) \leq b$, then vertical stacking $B_{k,l}$ to obtain $\mathbf{B}$, does not increase the row rank. Therefore, $\text{rank}(\mathbf{B}) \leq b$. 

\noindent$\mathbf{C:}$ Matrices $C_{k,l}$ are formed by $n$ times repeating the row vector $\text{diag}(\mathsf{B}^\intercal \Mmat_{k,l}^\intercal \Mmat_{k,l} \mathsf{B})^\intercal$. For rank considerations, we can omit all constant factors, and all copies of the row vector, except the first one. Stacking them vertically for all view pairs then yields the form,
\begin{align}
\label{eq:rank_pairs_C}
\begin{split}
\tilde{\mathbf{C}} = 
\begin{bmatrix}
\text{diag}(\mathsf{B}^\intercal (\Mmat_{1}^\intercal \Mmat_{1}-\Mmat_{2}^\intercal \Mmat_{2}) \mathsf{B})^\intercal \\
\vdots\\
\text{diag}(\mathsf{B}^\intercal (\Mmat_{m-1}^\intercal \Mmat_{m-1}-\Mmat_{m}^{\intercal} \Mmat_{m}) \mathsf{B})^\intercal \\
\end{bmatrix}.
\end{split}
\end{align}
Denoting the matrix difference $\tilde{\Mmat}_{k,l} = \Mmat_{k}^\intercal \Mmat_{k}-\Mmat_{l}^\intercal \Mmat_{l}$,  $\tilde{\mathbf{C}}$ can be written as follows,
\begin{align}
\label{eq:rank_pairs_C_unit2}
\begin{split}
\tilde{\mathbf{C}} \!=\! 
\begin{bmatrix}
\mathsf{e}_1^\intercal \mathsf{B}^\intercal \tilde{\Mmat}_{1,2} \mathsf{B} \mathsf{e}_1 & \dots & \mathsf{e}_n^\intercal \mathsf{B}^\intercal \tilde{\Mmat}_{1,2} \mathsf{B} \mathsf{e}_n \\
\vdots & \ddots & \vdots \\
\mathsf{e}_1^\intercal \mathsf{B}^\intercal \tilde{\Mmat}_{m-1,m} \mathsf{B} \mathsf{e}_1 & \dots & \mathsf{e}_n^\intercal \mathsf{B} \tilde{\Mmat}_{m-1,m}\mathsf{B} \mathsf{e}_n \\
\end{bmatrix},
\end{split}
\end{align}
where $\mathsf{e}_i \in \mathbb{R}^n$ are corresponding unit vectors. 

The above form reveals that we get a linearly dependent row i, if $\exists \, \alpha_j, \{M_j\}: \, \tilde{\Mmat}_{i} = \sum_{j \neq i} \alpha_j \tilde{\Mmat}_{j}$.
The rank of $\tilde{\mathbf{C}}$ is therefore bounded by the degrees of freedom of matrices $\tilde{\Mmat}_{k,l} = \Mmat_{k}^\intercal \Mmat_{k}-\Mmat_{l}^\intercal \Mmat_{l}$. Due to its symmetric property, the degrees of freedom amounts to the number of upper triangular elements of $\tilde{\Mmat}_{k,l} \in \mathbb{R}^{b \times b}$, which is given by $\frac{b(b+1)}{2}$.

\noindent$\mathbf{D:}$ In summary, the rank of $\mathbf{D}$ is bounded by
\begin{align}
\label{eq:rank_pairs_C_unit}
\begin{split}
\text{rank}(\mathbf{D}) &\leq \text{rank}(\mathbf{A}) + \text{rank}(\mathbf{B}) +\text{rank}(\mathbf{C}) \\
& \leq 1 + b + \frac{b(b+1)}{2} = \frac{(b+1)(b+2)}{2}.
\end{split}
\end{align}

\hfill \qed 
\vspace{-3pt}
\end{proof}

\section{Experimental Setting}  \label{sec:exp}
As in the main paper, we denote our method with motion embedding network as \emph{ours w/ motion}, and we denote the pipeline where the rigidity weights are explicitly set to 1 as \emph{ours w/o motion}. An overview of the architecture of our model consisting of a depth network and a motion embedding network is illustrated in Fig.~\ref{fig:architecture}.

\subsection{Network Architecture}
\begin{figure}[!htb]
  \centering
    \includegraphics[width=1\textwidth]{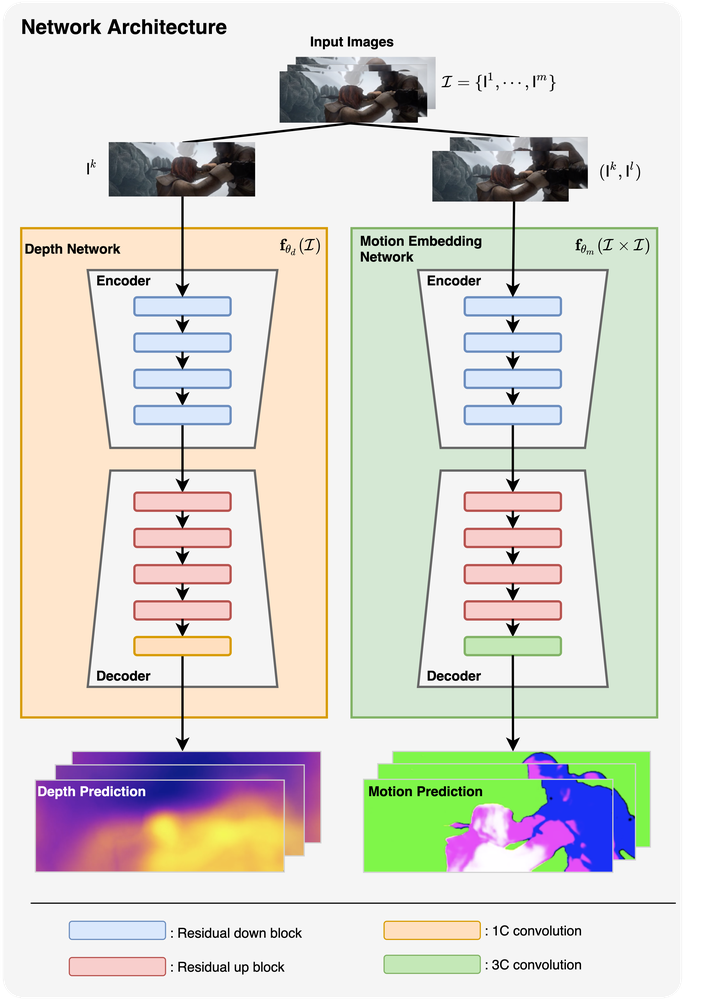}\\[-5pt]
    \caption{\textbf{Network architecture.} The architecture of our model consists of a depth network and a motion-embedding network, both of which incorporates a ResNet-18 based encoder and a decoder. The motion embedding network takes the concatenation of two images as its input whereas the depth network takes a single image as input.} 
    \label{fig:architecture}
\end{figure}

\boldparagraph{Depth Network.}
The depth prediction network, $\mathbf{f}_{\theta_d}(\mathcal{I})$, consists of a ResNet-18 \cite{resnet} based encoder and a decoder. The input to our depth network is a single RGB image, and the output from the network is a single depth map of the same size.  We have local skip connections within the encoder. This encoder has approximately 11M parameters. The depth decoder consists of residual up-blocks. Further details about the architecture of our depth decoder are given in Table~\ref{tab:depthdecoder}. The output from the depth decoder is passed through a sigmoid layer, whose output is then interpreted as scaled inverse depth. Scaled inverse depth is converted to depth using a linear map that converts the values in the range of [0,1] to the interval of $(d_{min}, d_{max})$, where $d_{min}$ is set to 0.1 meters, and $d_{max}$ is set to the maximum available depth value, separately for each different input video, for the test-time training of \textit{MPI Sintel} and \textit{VolumeDeform} datasets.

\begin{table}[!htb]
\scriptsize
\centering
\renewcommand{\tabcolsep}{6.5pt}
\begin{tabular}{ccccccc}
\thickhline 
\multicolumn{7}{c}{\textbf{Depth Decoder}} \\ \hline
\multicolumn{1}{c}{\textbf{layer}} & \multicolumn{1}{c}{\textbf{k}} & \multicolumn{1}{c}{\textbf{s}} & \multicolumn{1}{c}{\textbf{channels}} & \multicolumn{1}{c}{\textbf{res}} & \multicolumn{1}{c}{\textbf{input}} & \multicolumn{1}{c}{\textbf{activation}} \\ \thickhline 
upconv5 & 3  & 1  & 256  & 32   & econv5 & ELU  \\
iconv5  & 3  & 1  & 256  & 16   & upconv5, econv4 & ELU  \\ \hline
upconv4 & 3  & 1  & 128  & 16   & iconv5 & ELU  \\
iconv4  & 3  & 1  & 128  & 8    & upconv4, econv3 & ELU  \\
\hline
upconv3 & 3  & 1  & 64   & 8    & iconv4 & ELU  \\
iconv3  & 3  & 1  & 64   & 4    & upconv3, econv2 & ELU  \\
 \hline
upconv2 & 3  & 1  & 32   & 4    & iconv3 & ELU  \\
iconv2  & 3  & 1  & 32   & 2    & upconv2, econv1 & ELU  \\
\hline
upconv1 & 3  & 1  & 16   & 2    & iconv2 & ELU  \\
iconv1  & 3  & 1  & 16   & 1    & upconv1    & ELU  \\
disp1   & 3  & 1  & 1    & 1    & iconv1 & Sigmoid  \\ \thickhline
\end{tabular}
\vspace{-6pt}
\caption{\textbf{Depth Decoder Architecture.}}
\label{tab:depthdecoder}
\end{table}

\boldparagraph{Motion Embedding Network.}
The motion embedding network has a similar architecture as the depth network. The architecture of motion encoder is again a ResNet18-based encoder; however, different than the depth encoder, the motion encoder takes an input of 6 channels, corresponding to the concatenated channels of the two input images. The motion decoder architecture is provided in Table~\ref{tab:motiondecoder}. The output from the motion decoder is passed through a sigmoid to obtain a map of \textit{motion embeddings}. The architecture of the motion decoder is the same as the depth decoder, except for the last layer \textit{pix1}, where we have a 3-channel output instead of a 1-channel output, obtained as the output from a 3-channel convolutional layer. The number of output channel refers to the size of the pixel-wise motion embeddings. While we have employed motion embedding vectors of size 3, the number of output channels can be arbitrarily changed.

\begin{table}[!htb]
\scriptsize
\centering
\renewcommand{\tabcolsep}{6.5pt}
\begin{tabular}{ccccccc}
\thickhline 
\multicolumn{7}{c}{\textbf{Motion Embedding Decoder}} \\ \hline
\multicolumn{1}{c}{\textbf{layer}} & \multicolumn{1}{c}{\textbf{k}} & \multicolumn{1}{c}{\textbf{s}} & \multicolumn{1}{c}{\textbf{channels}} & \multicolumn{1}{c}{\textbf{res}} & \multicolumn{1}{c}{\textbf{input}} & \multicolumn{1}{c}{\textbf{activation}} \\ \thickhline 
upconv5 & 3  & 1  & 256  & 32   & econv5 & ELU  \\
iconv5  & 3  & 1  & 256  & 16   & upconv5, econv4 & ELU  \\ \hline
upconv4 & 3  & 1  & 128  & 16   & iconv5 & ELU  \\
iconv4  & 3  & 1  & 128  & 8    & upconv4, econv3 & ELU  \\
\hline
upconv3 & 3  & 1  & 64   & 8    & iconv4 & ELU  \\
iconv3  & 3  & 1  & 64   & 4    & upconv3, econv2 & ELU  \\
 \hline
upconv2 & 3  & 1  & 32   & 4    & iconv3 & ELU  \\
iconv2  & 3  & 1  & 32   & 2    & upconv2, econv1 & ELU  \\
\hline
upconv1 & 3  & 1  & 16   & 2    & iconv2 & ELU  \\
iconv1  & 3  & 1  & 16   & 1    & upconv1    & ELU  \\
pix1   & 3  & 1  & 3    & 1    & iconv1 & Sigmoid                                  \\ \thickhline
\end{tabular}
\vspace{-6pt}
\caption{\textbf{Motion Embedding Decoder Architecture.}}
\label{tab:motiondecoder}
\end{table}

\subsection{Training Procedure}

All experiments and case studies in this work were implemented using the publicly available PyTorch 1.4.0 distribution \cite{pytorch} in Python 3.6, using Cuda 9.0.
In all experiments, the ResNet18-based depth encoder and motion-embedding encoder, were initialized with weights from ImageNet \cite{imagenet} pre-training. Adam \cite{adam} optimizer with  $\beta_1= 0.9$,  $\beta_2=0.999$ was used, in combination with a learning rate decay by 0.1 every 15 epochs. As the small volume of \textit{MPI Sintel} and \textit{VolumeDeform} datasets does not yet enable us to train a depth network that can generalize, we have performed test-time training for each image sequence from these two datasets, for each method. In that direction, we trained the models from each baseline method as well as our method by only using a single sequence, and evaluated the depth-reconstruction performance on that sequence. In this test-time training pipeline, we only used the network as a non-linear function to estimate the depth, and used training as a procedure for optimizing the parameters of this network as a way of solving for the depth of each frame over a sequence. In our experiments with \textit{Hamlyn} however, we showed our generalization results, and utilized training/validation/test splits in the usual manner. For training with the \textit{Hamlyn} dataset, we used a subset of 5000 images from the provided training split, and evaluated the method on the provided test split. 

We followed a training schedule consisting of two stages. In the first stage, we jointly trained the motion-embedding network and the depth network. In the second stage, we froze the weights of the motion-embedding network, and continued the training of the depth network. In the latter stage of the training, we applied a $\tau$-offset to the weights. Here, the positive scalar $\tau$ represents the rigidity threshold as previously explained. 

For the sequences from the MPI Sintel dataset we have performed the first stage for 20 epochs, and the second stage for 50 or 60 epochs depending on the sequence. For the sequences from the VolumeDeform datatset, we performed the first stage of training for 10 epochs, and the second stage for 50 epochs. We applied the weight offset during the second stage of this training scheme.

\subsection{Detailed Explanation of Algorithm 1 and Implementation Details}
\setcounter{algorithm}{0}  %
\begin{algorithm}[ht]
\caption{
$[ \mathcal{L}^{kl}_\theta]=$\textbf{computeLossARAP}$(\mathsf{I}^{k}, \mathsf{I}^{l}$)}\vspace{0.5mm}
\begin{algorithmic}[1]

   \State Sample a set of edges $\mathcal{S}^k$ with their vertices $\mathcal{V}^k$.

 \State Estimate  the motion embedding $\Mmat^{kl}=\mathbf{f}_{\theta_m}(\mathsf{I}^{k}, \mathsf{I}^{l})$.
  
  \State Compute $\mathsf{W}^{kl}_\mathcal{S}$ using $\Mmat^{kl}$ and Fig.~(4) for $\mathcal{S}^k$.
  
  \State Establish $(i,j)$ between $(k,l)$ using  $\mathsf{F}^{kl} = \mathbf{f}_{\theta_f}(\mathsf{I}^k, \mathsf{I}^l)$.
  
  \State Start loop $s = k,l$
  \State -- Estimate the depth $\mathsf{\Lambda}^s=\mathbf{f}_{\theta_d}(\mathsf{I}^{s})$
  
 \State -- Reconstruct 3D 
 ${{\mathsf{X}^s}(\mathsf{\Lambda^s})=[\lambda_1^s\mathsf{u}_1^s,\ldots,\lambda_n^s\mathsf{u}_n^s]}$ for $\mathcal{V}^s$.
 
 \State --  Compute the EDM $\mathsf{E}(\mathsf{\Lambda}^s)$ using ${\mathsf{X}^s}(\mathsf{\Lambda^s})$.
 \State End loop
 
 \State Compute loss $\mathcal{L}_\theta(\mathsf{\Lambda}^k, \mathsf{\Lambda}^l, \mathsf{W}_\mathcal{S}^{kl})$ using \eqref{eq:loss-iso-nrsfm}.

\State Return $\mathcal{L}_\theta(\mathsf{\Lambda}^k, \mathsf{\Lambda}^l, \mathsf{W}_\mathcal{S}^{kl})$.
\end{algorithmic}
\algrule
Loss computation between two views for ARAP prior.
\end{algorithm}

In this section we elaborate on the steps of which the Algorithm 1 comprises, and to clarify any modifications to the main algorithm during the training stage.

Please recall our loss formulation:
\begin{equation}
\mathcal{L}_\theta(\mathsf{\Lambda}^k, \mathsf{\Lambda}^l, \mathsf{W}_\mathcal{S}^{kl}) =
 \frac{\norm{\mathsf{W}_\mathcal{S}^{kl}\odot\big(\mathsf{E}(\mathsf{\Lambda}^k)-\mathsf{E}(\mathsf{\Lambda}^l)\big)}_{1,1}}
      {\alpha\norm{\mathsf{W}_\mathcal{S}^{kl}}_{1,1}} \;.
 \label{eq:loss-iso-nrsfm}              %
\end{equation}
First of all, the calculations are computed in their vector forms rather than matrix forms due to computational considerations. In our implementation, we use consecutive images as view pairs, \ie $(k,l)$ is computed as $(k, k+1)$. 

To start the algorithm, a set of edges $\mathcal{S}^k$ with their vertices $\mathcal{V}^k$ are sampled. In that regard, from the set of all points in view $k$, we filter out the points which do not have correspondences in the view $l$. From the remaining points, we sample a set of point pairs for a given sample size. In our experiments we sample 100k point pairs for which both points have correspondences in the corresponding view. 

We use the images $\mathsf{I}^{k}$ and $\mathsf{I}^{l}$ and concatenate them in their channel dimensions. We pass this concatenated tensor through the motion embedding network and obtain a map of motion embeddings, \ie $\mathsf{M}^{kl}=\mathbf{f}_{\theta_m}(\mathsf{I}^{k}, \mathsf{I}^{l})$.

Using the list of sampled point pairs, which denotes the sampled edges in our formulation, we retrieve the motion embedding vectors associated with each point in each pair. We obtain a list of motion embedding pairs, among which we calculate the pairwise distances. For each embedding pair, say $\mathsf{m}_i^{kl}, \mathsf{m}_j^{kl}$, we calculate the associated weight $w_{ij}^{kl}$ via $w_{ij}^{kl} = 1 - tanh(\norm{\mathsf{m}_i^{kl} - \mathsf{m}_j^{kl}})$. After having calculated this value for each pair, we have a list of rigidity scores, which we use for re-weighting our pair-wise rigidity loss. In the second stage of the training, we also apply the previously mentioned $\tau-$offset to the weights in the following form: $\tilde{w_{ij}^{kl}} = \frac{w_{ij}^{kl} + \tau}{1 + \tau}$, followed by the clamping of these values to $[0,1]$. In Fig.~\ref{fig:offset}, we demonstrate the rescaled weight values obtained using different $\tau$ values.

 \begin{figure}[ht]
  \centering
    \includegraphics[width=0.8\textwidth]{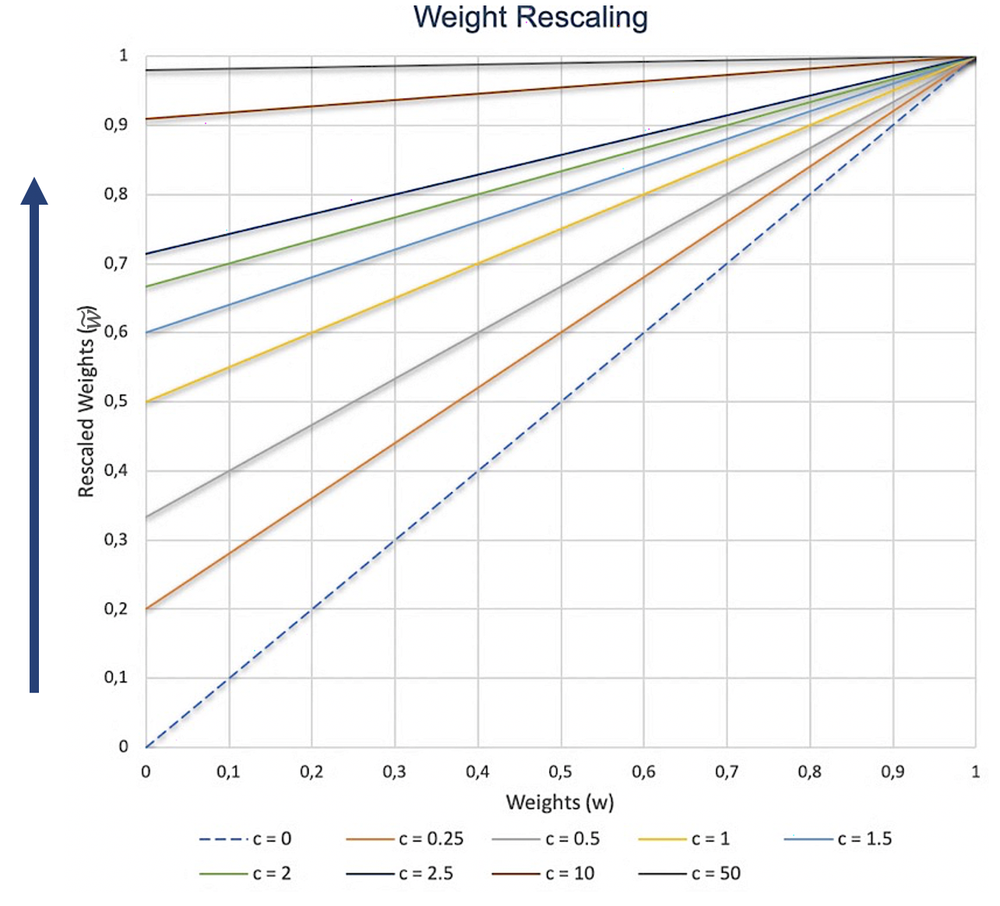}\\[-8pt]
    \caption{\textbf{Weight rescaling with offsets.} For each different offset value $\tau$, the linear rescaling of weights.}
    \label{fig:offset}
\end{figure}

In parallel to the computation of the motion embeddings and rigidity scores, we compute the depth predictions to formulate our loss based on our ARAP formulation. We use the images $\mathsf{I}^{k}$ and $\mathsf{I}^{l}$, and separately pass them through the depth network to obtain depth predictions $\mathsf{\Lambda}^k = \mathbf{f}_{\theta_d}(\mathsf{I}^{k})$ and $\mathsf{\Lambda}^l = \mathbf{f}_{\theta_d}(\mathsf{I}^{l})$. 
 
To compute the pairwise distances, we need to have the corresponding point pair for each pair we sampled. For each point pair $(X_i^k, X_j^k)$ in our list of sampled edges, we compute the corresponding point pair $(X_i^l, X_j^l)$ in view $l$ using the dense correspondences $\mathsf{F}^{kl} = \mathbf{f}_{\theta_f}(\mathsf{I}^k, \mathsf{I}^l)$. This results in another list of point pairs, which are all represented by their homogeneous image coordinates. Using the predicted dense depth maps, we compute the 3D point coordinates for each point as ${{\mathsf{X}^k}(\mathsf{\Lambda^k})=[\lambda_1^k\mathsf{u}_1^k,\ldots,\lambda_n^k\mathsf{u}_n^k]}$ for $\mathcal{V}^k$ and ${{\mathsf{X}^l}(\mathsf{\Lambda^l})=[\lambda_1^l\mathsf{u}_1^l,\ldots,\lambda_n^l\mathsf{u}_n^l]}$ for $\mathcal{V}^l$. 
 
For computational purposes, we compute the EDM for the points in its vector form $\mathsf{e}_\mathcal{S}^k$ and $\mathsf{e}_\mathcal{S}^l$, representing the pairwise distances for the sampled pairs in view $k$, and the corresponding pairs in view $l$. $\mathsf{e}_\mathcal{S}^k$ is a vector whose entries, $e_{\mathcal{S},ij}^k$, represent the squared Euclidean distances between the points within a pair $(i,j)$ such that $(X_i,X_j) \in \mathcal{S}$. The squared Euclidean distances between the corresponding point pair in the $l$-th view are denoted by $\mathsf{e}_{\mathcal{S},ij}^l$. 
Using these two distance vectors, we compute the loss for a given pair of views $(k,l)$, via $\mathcal{L}_\theta(\mathsf{\Lambda}^k, \mathsf{\Lambda}^l, \mathsf{W}_\mathcal{S}^{kl})$ using \eqref{eq:loss-iso-nrsfm}, however in a modified form suiting the vector structure. Similarly, we define $\mathsf{w}_\mathcal{S}^{kl}$, denoting the motion similarity weights in a vector form, where $w_{\mathcal{S},ij}^{kl}$ terms are the non-zero entries from $\mathsf{W}_{\mathcal{S}}^{kl}$, which is how we make use of the sparsity of this weight matrix as we mentioned in the main paper. 

In our implementation, we employ several normalization schemes and regularization terms. Here, we elaborate on these modifications we employ while forming our loss function in implementation. 
First of all, it is important to note that with the objective formulation from~\eqref{eq:loss-iso-nrsfm}, there exists a trivial solution for the minimization problem, which can be attained via predicting zero depth at each point, unless we use the normalization term $\alpha$. During the training, without any sort of depth or distance normalization, the network quickly starts to make constant, zero depth prediction independent of the given input image. With the aim of mitigating this problem, we normalize the pairwise distances via the following formulation: $\tilde{e_{\mathsf{S},ij}^k} = \frac{e_{\mathsf{S},ij}^k}{\sum\limits_{\substack{(i,j):(X_i, X_j)\in S}}{e_{\mathsf{S},ij}^k}}$, and similarly for $\tilde{e_{\mathsf{S},ij}^l}$. Additionally, we compute the normalization factor $\alpha$ from~\eqref{eq:loss-iso-nrsfm} via $\alpha=\norm{\mathsf{e}(\mathsf{\Lambda}^k)+\mathsf{e}(\mathsf{\Lambda}^l)}_{1}$ for numerical stability and for avoiding reconstructions with near-zero depth values. We additionally incorporate a weight-norm regularization term into our training objective, \ie, $\beta\norm{\mathsf{w}_\mathcal{S}^{kl}}_{1}$, to control the maximization of the motion similarity scores (weights),
where $\beta$ is a coefficient determining the contribution of the weight-norm regularization term, L1-norm of the weight vector, to the overall training objective. In our experiments, we use $\lambda=0.01$.

\subsection{Runtimes}
For our experiments with test-time optimization, \ie for Sintel and VolumeDeform (VD), runtimes vary based on the image dimensions and sequence length.
Runtimes are 15-45m for Sintel and 40m-1h30m for VD.
Inference times on Hamlyn are significantly smaller ($<$0.1s per frame) where we report generalization results by first training a model and then directly inferring the depth by a single forward-pass. For Hamlyn, please see the Table~\ref{tab:thresholds}.
\begin{table}[h]
  \centering
  \scriptsize
  \setlength{\tabcolsep}{2.7pt}
  \renewcommand{\arraystretch}{1.1}
  {\begin{tabular}{lrrrrrrr}
    \toprule
    Runtimes      & PackNet& Li & DLH  & MaxRig & MDH & Ours(w/) &Ours(w/o) \\  %
    \hline
   Training & 250m & 980m & - & - & - & 110m & 248m\\ %
    Inference(total) & 7m & 7m & 137m & 790m & 210m & 13m & 13m \\
    \bottomrule
  \end{tabular}}
  \caption {\textbf{Runtime comparison for the Hamlyn dataset.} Runtime values for training and inference across difference methods for the Hamlyn dataset.}
  \label{tab:thresholds}
\end{table}

\subsection{Stopping criteria}
Our stopping criteria are based on the geometrical loss and the number of epochs, and not on the evaluation metrics. We stop once the loss deviations are within a margin. We have a rather high margin for these deviations as the computed loss always depends on the randomly sampled point pairs. For the generalization experiments (Hamlyn), we rely on the validation loss as the stopping criterion.

\subsection{Comparison with Other Methods}
In this section, we elaborate on how we trained our baseline models.
The results from unsupervised monocular pipelines \textit{PackNet} \cite{packnet} and \textit{Li}~\etal~\cite{google-depth}, were obtained by training a model separately for each sequence from \textit{MPI Sintel} and \textit{VolumeDeform}. For the \textit{Hamlyn} dataset, we trained the models using the training split from \textit{Hamlyn}. 
The results from the non-rigid reconstruction models \textit{DLH} \cite{dai2012simple}, \textit{MDH} \cite{probst2018incremental} and \textit{MaxRig} \cite{maxrig} were obtained using the complete test sequence to perform reconstruction, as these methods are not single-view methods and they require a set of images to solve for the depth. It is also important to note that \textit{DLH} \cite{dai2012simple} and \textit{MaxRig} \cite{maxrig}  methods provide sparse reconstructions. For our comparisons, we scaled each depth map prediction by using the median of the points for which a prediction has been made. We evaluated the performance of these models by computing the errors only for these points.

\section{Evaluation Methodology}

\subsection{Evaluation Metrics for Depth Reconstruction}
As unsupervised monocular depth reconstruction has an inherent scale ambiguity, the depth reconstruction is up to an unknown scale factor. Previous approaches have followed an evaluation strategy of rescaling the predicted depth maps prior to evaluation. While there are numerous ways of performing this rescaling, we follow the approach from \cite{zhou} and \cite{digging}. For each view $k$, we use the ground truth depth map $\mathsf{\Lambda}_{gt}^k$, and the depth prediction $\mathsf{\Lambda}^k$, to calculate the scaling factor  $\hat s$ as
\begin{equation}
 \tilde s =\frac{\mathrm{median}(\mathsf{\Lambda}_{gt}^k)}{\mathrm{median}(\mathsf{\Lambda}^k)}.
\end{equation}
We then calculate the rescaled depth map $\tilde{\mathsf{\Lambda}^k}$ as
\begin{equation}
 \tilde{\mathsf{\Lambda}^k} = \tilde s  \mathsf{\Lambda}^k.
\end{equation}
We perform this median depth rescaling step only prior to evaluation. Then, we assess the performance of our depth reconstruction pipeline by using metrics widely employed in current depth reconstruction benchmarks detailed below.

For a rescaled depth map $\tilde{\mathsf{\Lambda}^k}$, and ground-truth depth map $\mathsf{\Lambda}_{gt}^k$, we use the indexing $\tilde{\lambda_{ij}^k}$ and $\lambda_{gt,ij}^k$ to calculate the following error metrics. \\

\boldparagraph{Mean Absolute Relative Error \cite{make3d}.}
\begin{equation}
ARel = \frac{1}{n}\sum_{i,j} \frac{\left| \lambda_{gt,ij}^k - \lambda_{ij}^k\right|}{\lambda_{gt,ij}^k}
\end{equation}

\boldparagraph{Linear Root Mean Square Error (RMSE) \cite{towardsholisticscene}.}
\begin{equation}
RMSE = \sqrt{\frac{1}{n}\sum_{i,j} {\left(\lambda_{gt,ij}^k - \lambda_{ij}^k\right)^2}}
\end{equation}

\boldparagraph{Log-scale Invariant RMSE \cite{eigen}.}
\begin{equation}
RMSE_{log} = \sqrt{\frac{1}{n}\sum_{i,j} \left(\log{\lambda_{gt,ij}^k} - \log{\lambda_{ij}^k}\right)^2}
\end{equation}\\

\boldparagraph{Accuracy under a threshold \cite{pullingthings}.}
\begin{equation}
\delta < \mathcal{T} =  \max{ \left(\frac{\lambda_{ij}^k}{\lambda_{gt,ij}^k},\frac{\lambda_{gt,ij}^k}{\lambda_{ij}^k}\right)}
\end{equation}
where $\mathcal{T}$ is a threshold, set to $\mathcal{T}=1.25$, $\mathcal{T}=1.25^2$, $\mathcal{T}=1.25^3$, for three different evaluation settings, as suggested in \cite{pullingthings}.

For the \textit{MPI Sintel} dataset, we evaluate the performance only where the ground-truth depth is smaller than 50 meters, except for the \textit{mountain\_1} sequence where we evaluate the performance where the ground-truth depth is smaller than 500 meters. For the other datasets, we evaluate the depth for every pixel where the ground-truth is available.

\subsection{Motion segmentation}
For the evaluation of the motion embedding pipeline, we use the \textit{Intersection-over-Union (IoU)} metric, also known as \textit{Jaccard index}, which represents the overlap percentage between the predicted labels and the ground truth labels.

\boldparagraph{Overall pixel accuracy.}
We calculate the overall pixel accuracy as the ratio between the number of correctly classified pixels and the total number of pixels in the image, as in

\begin{equation}
\mathrm{ACC} = \frac{\# \mathrm{correctly\;classified\;pixels}}{\# \mathrm{total\; pixels}} 
\end{equation}

\boldparagraph{Intersection-over-Union (IoU).}
\begin{equation}
\mathrm{IoU} = \frac{\left|GT \cap Pred\right|}{\left|GT \cup Pred \right|} = \frac{TP}{TP + FN + FP}
\end{equation}

\section{Dataset Details}
\subsection{Datasets}
In this section, we provide an overview of the datasets we use in this work. We perform our experiments and evaluate our method using three datasets, namely \textit{MPI Sintel} \cite{sintel}, \textit{VolumeDeform} \cite{volumedeform}, and \textit{Hamlyn} \cite{hamlyn} datasets. Due to the small volume of \textit{MPI Sintel} and \textit{VolumeDeform} datasets, we perform test-time training for the sequences from these datasets, where we train a separate model for each sequence. Table~\ref{tab:datasets} presents a comparison between some of the properties of \textit{MPI Sintel} and \textit{VolumeDeform} datasets.

\begin{table}[ht]
\centering
\small
\setlength{\tabcolsep}{4pt}
\begin{tabular}{lcc}
\thickhline 
\textbf{Dataset} & \textbf{MPI Sintel~\cite{sintel}} & \textbf{VolumeDeform~\cite{volumedeform}} \\ \hline
\# Training frames  & 1064  & 4806  \\ 
\# Training scenes & 25  & 8   \\
Resolution & 1024 $\times$ 436  & 640 $\times$ 480   \\ \hline 
Camera intrinsics & \yes  & \yes  \\
Camera extrinsics & \yes  & \yes  \\\hline 
Disparity/Depth & \yes  & \yes  \\
Optical flow & \yes  & \noo \\ \hline
Motion segmentation & \yes  & \noo    \\ \hline
\end{tabular}
\vspace{-6pt}
\caption {\textbf{Comparison between MPI Sintel \cite{sintel} and VolumeDeform \cite{volumedeform} datasets used in this work.}}
\label{tab:datasets}
\end{table}

\boldparagraph{MPI Sintel~\cite{sintel}:}
\textit{MPI Sintel} is a synthetic dataset that has been developed to address the limitations of optical flow benchmarks by providing naturalistic video sequences with motion blur, non-rigid motion and long-range motion \cite{sintel}. The dataset is based on modifications on the open source animated short film \textit{Sintel}, with the purpose of obtaining image sequences that could be useful for optical flow evaluation \cite{sintel}. The \textit{MPI Sintel} dataset contains information about properties of image sequences such as optical flow, depth, and motion boundaries \cite{sintel}. The image sequences are available in two categories, \textit{clean}, and \textit{final} passes, which are rendered with different difficulty levels \cite{sintel}. The main reasoning behind our preference for using this dataset is the availability of perfect ground truth optical flow for all of the sequences. This enables us to mitigate problems that could be caused by noisy correspondences, and consequently provides us with an experimental setting that is isolated from possible inaccuracies in optical flow estimations. Additionally, having ground truth depth information for each sequence creates an ideal setting for evaluating our depth estimation algorithm for dynamic and deforming scenes. It also captures a wide variety of scenes with non-rigid deformations and illumination changes, as well as image sequences of varying depth scales \cite{sintel}. 

For our experiments, taking into consideration the maximum available depth of the sequences, we selected a set of 14 sequences from the final-pass category of the training set of \textit{MPI Sintel}, namely \textit{alley\_1, alley\_2, ambush\_2, ambush\_4, ambush\_5, ambush\_6, cave\_2, cave\_4, market\_2, market\_5, mountain\_1, shaman\_3, sleeping\_1, sleeping\_2}.

\boldparagraph{VolumeDeform~\cite{volumedeform}:}
In order to perform further experiments on non-rigid deformations, we use the \textit{VolumeDeform} dataset, which consists of 8 separate video sequences, namely \textit{adventcalendar, boxing, hoodie, minion, shirt, sunflower, umbrella, upperbody}. These 30 Hz video streams are captured by an RGB-D sensor, recording at a resolution of $480 \times 640$ pixels \cite{volumedeform}.

\begin{figure}[ht]
  \centering
    \includegraphics[width=1\textwidth]{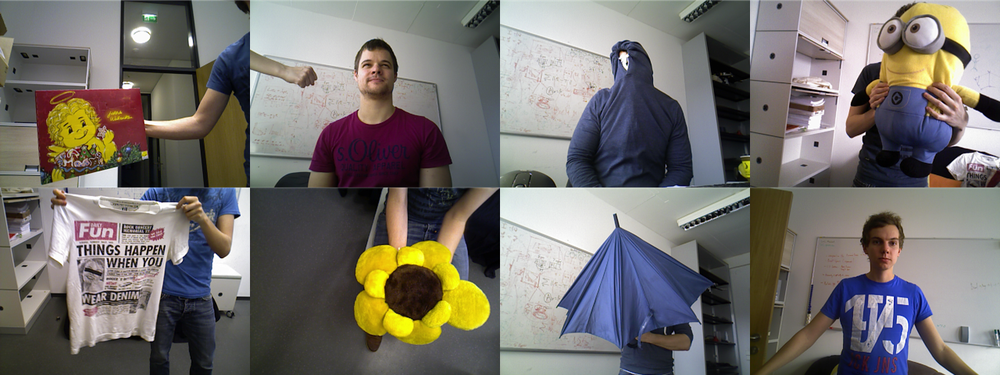}\\[-6pt]
    \caption{\textbf{Sample images from the \textit{VolumeDeform} dataset.} Example frames from VolumeDeform dataset. From left to right, top to bottom, the sequences: "adventcalendar", "boxing", "hoodie", "minion", "shirt" "sunflower", "umbrella", "upperbody".}
\end{figure}

\boldparagraph{Hamlyn Centre Laparoscopic Video Dataset~\cite{hamlyn}:}
In order to evaluate the generalization capability of our method, we use the \textit{Hamlyn Centre Laparoscopic Video Dataset}, which consists of approximately 40000 pairs of rectified stereo images collected from a partial nephrectomy surgery \cite{hamlyn}.  As we do not have access to the ground-truth depth for the \textit{Hamlyn} dataset, we use OpenSFM \cite{opensfm-1, opensfm-2} to obtain the depth through stereo reconstruction using the calibrated stereo pairs in the dataset.

\subsection{Obtaining Optical Flow}
In our work, we require having dense correspondences for our image pairs, as previously discussed. Although we are provided with ground-truth optical flow for the \textit{MPI Sintel} dataset, we do not have ground-truth correspondences for the \textit{VolumeDeform} and \textit{Hamlyn} datasets. For the \textit{VolumeDeform} dataset, we use \textit{RAFT} \cite{raft}, and for the \textit{Hamlyn} dataset we use \textit{DDFlow} \cite{ddflow} to predict dense optical flow. 

\begin{figure*}[!htb]
  \centering
  \multirow{1}{*}[10pt]{{\textbf{\quad\quad Input\quad\quad\quad\quad\quad\quad\quad\quad\quad\quad GT\quad\quad\quad\quad\quad\quad\quad\quad\quad\quad Li~\etal~\cite{google-depth} \quad\quad\quad\quad\quad\quad Ours w/ motion}}}
    \includegraphics[width=\textwidth]{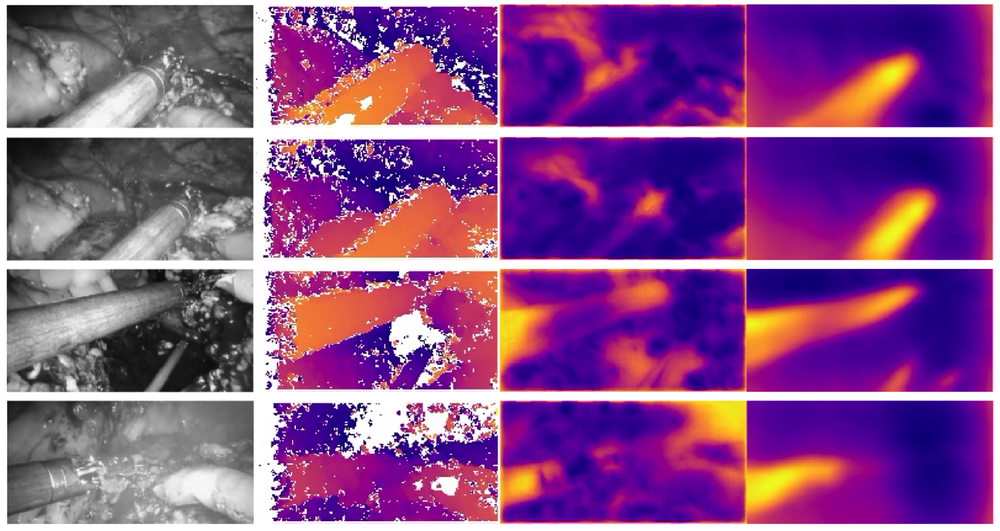}\\[-10pt]
    \caption{\textbf{Depth reconstruction for Hamlyn dataset~\cite{hamlyn}.} A comparison between the depth obtained from stereo-reconstruction (obtained using OpenSFM, labeled as GT), depth prediction results from Li~\etal~\cite{google-depth} and our method with motion embeddings are presented. Please note that the input images are represented in grayscale here, whereas they have been used as RGB images in our pipeline.}
    \label{fig:hamlynfull}
\end{figure*}

\subsection{Data pre-processing: } For the \textit{MPI Sintel} dataset, original images are of size $436 \times 1024$. The images were center-cropped to obtain input images of size $384 \times 1024$. For the \textit{VolumeDeform} dataset, the images are of size $480 \times 640$, and for the \textit{Hamlyn} dataset the image size is $192 \times 384$. For the \textit{VolumeDeform} and \textit{Hamlyn} datasets, we do not crop or resize the images, and we directly use the original image size. For all three of the datasets, we perform a standardization of the images based on the mean and variance of the images used during the ImageNet~\cite{imagenet} pre-training.

\section{Further Results and Analysis}

\boldparagraph{Depth Reconstruction.}
In Fig.~\ref{fig:hamlynfull}, we provided additional qualitative results for depth reconstruction. Additionally, we performed an analysis of the error distribution for the depth predictions from our model to better understand the error values. As we suspected, points with small ground truth depth values and for scenes with slightly wrong scale corrections often result in larger Mean Absolute Relative Error (ARE) values. In that direction we wanted to compare the ARE values for our model with the ARE values obtained with depth values that are uniform randomly sampled. For the randomly sampled depth values, we obtained ARE values that are larger by orders of magnitude. As shown in Fig.~\ref{fig:histograms}, which showcases the error distributions, our results are substantially better than uniform random depth predictions Figure~\ref{fig:histograms}. 

\begin{figure}[t]
  \centering
  \scriptsize
  \setlength{\tabcolsep}{1pt}
  \begin{tabular}{cc}  
     \includegraphics[width=0.54\columnwidth]{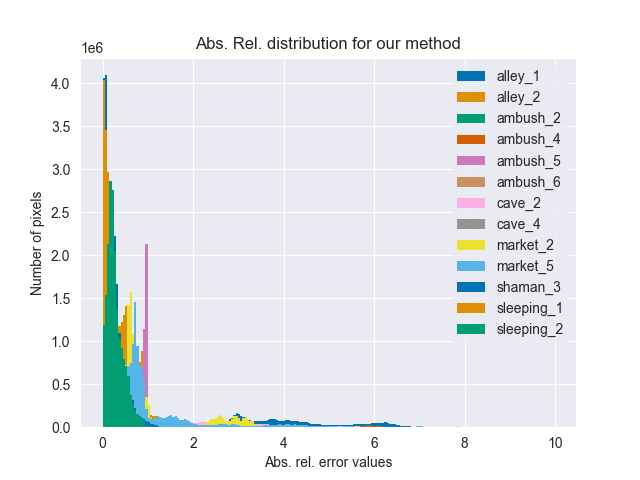} &
     \includegraphics[width=0.54\columnwidth]{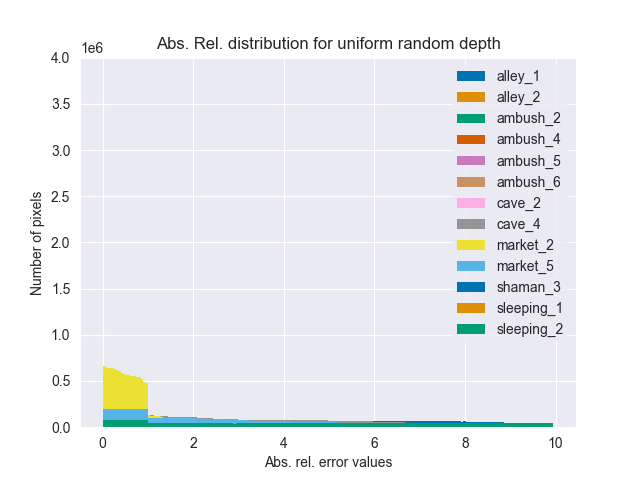} \\[-3pt]
     ours w/ motion, depth & uniform random depth \\[-9pt]
  \end{tabular}
  \caption{\textbf{\#Pixels vs. abs. rel. error histograms for the Sintel dataset.} 
  \textbf{Left:} depth predictions from ours w/ motion. \textbf{Right:} uniform random depth sampling. Please note the axis scales. Our errors are significantly better than random. }
  \label{fig:histograms}
  \vspace{2pt}
\end{figure}

\boldparagraph{Motion Segmentation.}
We evaluated the performance of our per-pixel motion embeddings. As the \textit{ground-truth} motion-embeddings are not available, we perform an evaluation based on the moving object segmentation task. We perform the motion segmentation evaluation only for the \textit{MPI Sintel} dataset, for which the ground-truth per-pixel motion segmentation maps are available \cite{gt-sintel-motsegm}.  In Table~\ref{tab:threshold}, we compare the performance of our model with two recent works \cite{motsegm-baseline} and \cite{loquercio}, and demonstrate our competitive results for the motion segmentation task in Fig.~\ref{fig:motsegmfull}. For obtaining motion segmentation results from our embeddings, we separate the static part from the dynamic parts of the scene by performing a thresholding based on the distance between per-pixel motion embeddings and an embedding vector representing the static cluster.  First, we obtain the channel-wise median of the embeddings corresponding to the pixels on the image borders across the whole image sequence. We assign this median embedding as the center of the cluster corresponding to the static background. By calculating the Euclidean distance between each of the embeddings and this cluster center, we obtain \textit{staticity} scores. By determining a distance threshold, pixels corresponding to the moving segments are separated. We report the scores with a distance threshold of $\mathcal{T} = 0.1$, in Table~\ref{tab:threshold}.

\begin{table*}[ht]
\footnotesize
\centering
\renewcommand{\tabcolsep}{5.6pt}
\begin{tabular}{ll|cccccccccccccc}
\thickhline 
  & & alley\_1 &  alley\_2 &  ambush\_2 & ambush\_4 & ambush\_5 & ambush\_6 & cave\_4& market\_2 &   market\_5  & shaman\_3  \\ \thickhline

Taniai~\etal~\cite{motsegm-baseline} & \textbf{ACC} $\uparrow$  & 0.942 & 0.978  & \textbf{0.983}  & 0.780  & \textbf{0.974} & 0.470 & \textbf{0.852} & \textbf{0.912} & \textbf{0.855}  & \textbf{0.974}     \\ %

&\textbf{IoU} $\uparrow$ & 0.856 & \textbf{0.738}  & \textbf{0.963}  & 0.531  &  \textbf{0.947} & 0.264 & \textbf{0.717} & \textbf{0.723}  & \textbf{0.713}   &  \textbf{0.829}   \\ %

Yang~\etal~\cite{loquercio} & \textbf{ACC} $\uparrow$  & 0.874 & 0.931  & 0.629  & 0.755  & 0.713 & 0.697 & 0.622 & 0.864  & 0.739 &0.732     \\ %

&\textbf{IoU} $\uparrow$ & 0.670 & 0.598  & 0.421  & 0.534  &  0.521 & 0.489 & 0.426 &0.612 & 0.517  & 0.406    \\ %

Ours& \textbf{ACC} $\uparrow$    & \textbf{0.947} & \textbf{0.980} & 0.971   & \textbf{0.862}  & 0.972  & \textbf{0.942}   & 0.797 & 0.881 & 0.842   & 0.921  \\

&\textbf{IoU} $\uparrow$ & \textbf{0.875} & 0.545  & 0.944  & \textbf{0.817}  &  0.945 & \textbf{0.881}   & 0.649 & 0.627 & 0.569 & 0.461  \\
\thickhline 
\end{tabular}
\vspace{-6pt}
\caption {\textbf{Quantitative analysis of motion segmentation performance on MPI Sintel~\cite{sintel}, (for the thresholding based motion segmentation method).}  Motion segmentation performance of our model was evaluated on MPI Sintel using Overall Pixel Accuracy (ACC) and Intersection-over-Union (IoU) metrics.}
\label{tab:threshold}
\end{table*}

\begin{figure}[H]
  \centering
    \includegraphics[width=1.02\columnwidth]{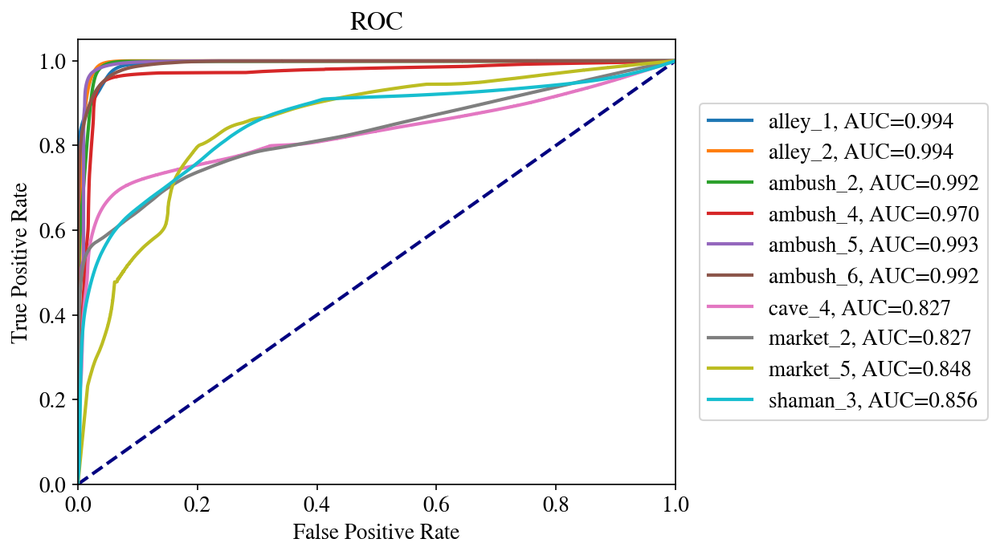}\\[-6pt]
    \caption{\textbf{ROC for motion segmentation performance on MPI Sintel \cite{sintel}.} ROC curves for each sequence are demonstrated.}
    \label{fig:roc}
\end{figure}

\begin{figure}[h]
  \centering
    \includegraphics[width=\textwidth]{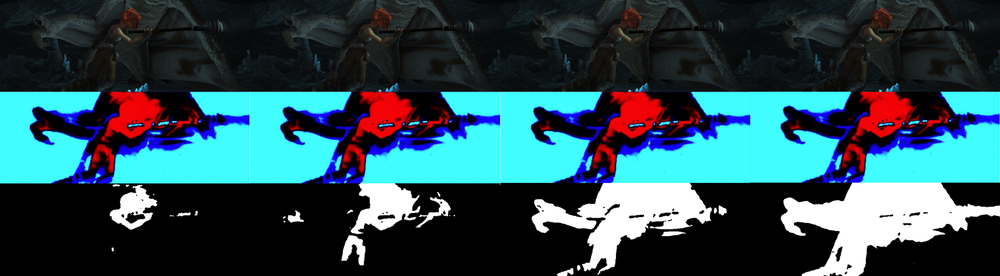}\\[-6pt]
    \caption{\textbf{Moving object segmentation results for different distance thresholds.} As we decrease the distance threshold (from left to right), we can see the variation of the the motion segmentation results. This image demonstrates that our embeddings do not simply provide a motion segmentation, but also provide information about different motion clusters.}
    \label{fig:thresholds}
\end{figure}

\begin{figure*}[!h]
  \centering
  \multirow{1}{*}[10pt]{{\textbf{\; 
  Image\qquad\qquad\qquad\qquad 
  GT\qquad\qquad\quad
  Taniai~\etal~\cite{motsegm-baseline} \qquad\quad Yang~\etal~\cite{loquercio} \qquad\qquad\quad
  Ours}}}
    \includegraphics[width=\textwidth,trim={0 0 0 1cm},clip]{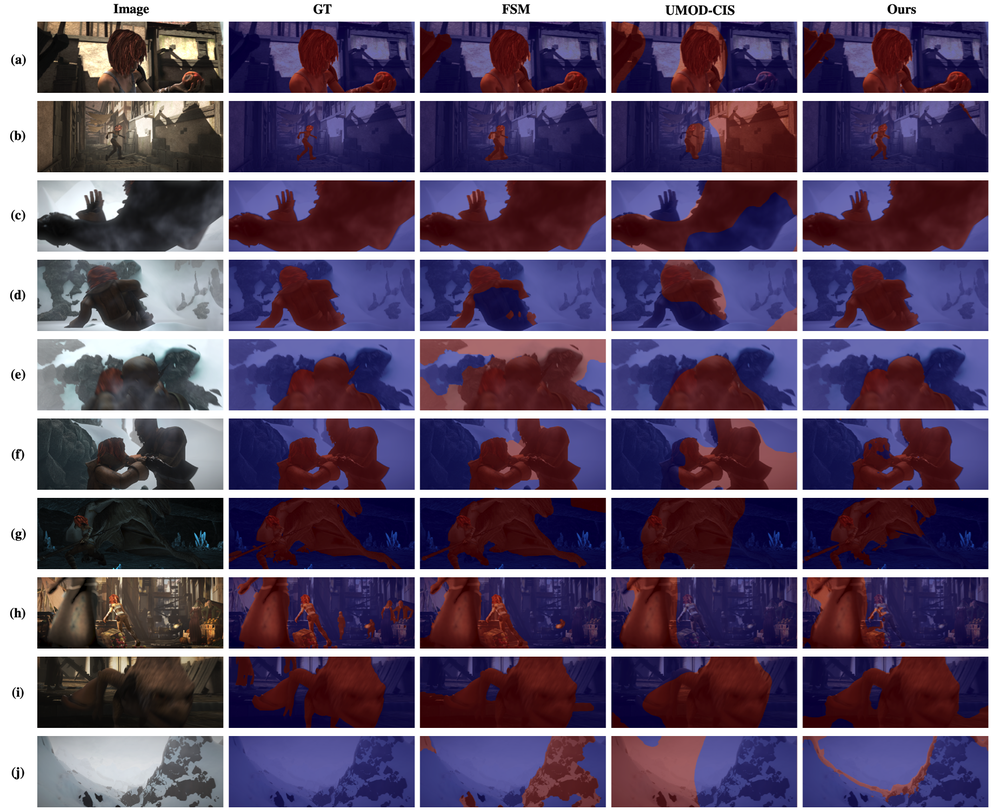}\\[-6pt]
    \caption{\textbf{Motion segmentation performance.} We compare our results with the ground-truth and results from two other methods. Our method gives competitive results despite motion segmentation being only an auxiliary output in our pipeline. }
    \label{fig:motsegmfull}
\end{figure*}

We elaborate more on what this thresholding means and what happens to the segmentation map as we vary the threshold. With this purpose, we have plotted the ROC curve for motion segmentation performance in Fig.~\ref{fig:roc}. In Fig.~\ref{fig:thresholds}, we demonstrate how this distance threshold affects the motion segmentation results and how different motion clusters start being classified as dynamic as we vary the threshold. With this observation, we want to highlight that our motion embeddings provide not just a separation between the static and the dynamic parts, bu they also represent different motions in the image pair.

\boldparagraph{\newline Limitations.}
\textit{Dependency on optical flow:}
One of the important aspects of our method is that our predictions are dependent on the quality of the correspondences, as our geometric loss assumes known correspondences across views. As previously explained, we use three different ways to obtain dense correspondences: \textbf{(1)} obtained in an unsupervised manner (for the Hamlyn dataset), \textbf{(2)} in a supervised manner (for the VolumeDeform dataset), and \textbf{(3)} ground truth (for the MPI Sintel dataset).  Note that despite the fact that the quality of optical flow maps obtained from these three methods is different, meaningful depth reconstruction results we obtained demonstrate the robustness of our method with respect to the quality of provided optical flow maps. However, in addition to this remark, we also would like to highlight that with increasing correspondence noise, we expect to observe a degradation in the depth prediction quality. This is an issue that needs to be paid attention, particularly in highly complex and dynamic real-life scenarios, where the optical flow estimation obtained from pre-trained networks could be erroneous.

\end{document}